\RequirePackage[l2tabu, orthodox]{nag}
\documentclass[12pt]{article}

\pdfoutput=1

\usepackage[parfill]{parskip}
\usepackage[margin=1in]{geometry}
\usepackage[defaultlines=3,all]{nowidow}

\usepackage[T1]{fontenc}
\usepackage{microtype}
\usepackage[english]{babel}
\usepackage{libertine}             
\usepackage{titlesec}
\titleformat*{\section}{\Large\bfseries}

\usepackage{amsmath}
\usepackage{amssymb}
\usepackage{amsthm}
\usepackage{mathtools}

\usepackage{amsfonts}       
\usepackage{nicefrac}       
\usepackage{microtype}      
\usepackage{mathrsfs}
\usepackage{mathabx}

\usepackage[dvipsnames]{xcolor}

\usepackage[sort]{natbib}
\usepackage{bibunits}

\usepackage{hyperref}
\hypersetup{colorlinks,linktoc=all}
\usepackage[all]{hypcap}
\hypersetup{citecolor=Violet}
\hypersetup{linkcolor=black}
\hypersetup{urlcolor=MidnightBlue}

\usepackage{xurl}   

\usepackage[nameinlink]{cleveref}

\usepackage{booktabs}       
\usepackage[inline]{enumitem}
\usepackage{bm,bbm}
\usepackage{fontawesome}
\usepackage{setspace}
\usepackage{dsfont}

\usepackage{caption}
\usepackage{subcaption}
\usepackage{appendix} 
\usepackage{pgfplots}
\usepackage{tikz}
\usepackage{tikzpeople}
\usepgflibrary{arrows}
\usetikzlibrary{shapes,shadows,decorations,arrows,calc,arrows.meta,fit,positioning}
\usetikzlibrary{fit,backgrounds}
\usetikzlibrary{shadows.blur}
\usetikzlibrary{shadings,intersections}
\usetikzlibrary{snakes}

\newtheorem{definition}{Definition}
\newtheorem{theorem}{Theorem}

\newtheorem{proposition}{Proposition}
\newtheorem{example}{Example}
\newtheorem{remark}{Remark}
\newtheorem{lemma}{Lemma}

\bibpunct[; ]{(}{)}{,}{a}{}{;}

\title{(Im)possibility of Collective Intelligence}

\author{
Krikamol Muandet\thanks{
The author is indebted to Bernhard Sch\"olkopf, Moritz Hardt, Kate Larson, Elias Bareinboim, Arash Mehrjou, Marina Munkhoeva, Adrian Javaloy, Thiparat Chotibut, Sebastian Stich, Samira Samadi, Pasin Manurangsi, Celestine Mendler-Duenner, Isabel Valera, Yassine Nemmour, Bruno Kacper Mlodozeniec, Simon Buchholz, Junhyung Park, Felix Leeb, Nasim Rahaman, Heiner Kremer, Emtiyaz Khan, Gill Blanchard, Laurent Chaplin, and Aritz P\'{e}rez for fruitful discussions and constructive feedback. The author also thanks
the participants of the seminars at Max Planck Institute for Intelligent Systems (MPI-IS), Bilbao Workshop on The Mathematics of Machine Learning, Korea Institute for Advanced Study (KIAS), RIKEN-AIP, and Mila - Quebec AI Institute for raising several thought-provoking questions about this work.
Finally, the author thanks anonymous reviewers whose feedback has significantly improved the manuscript.
} \\
CISPA Helmholtz Center for Information Security \protect \\ Stuhlsatzenhaus 5, 66123 Saarbr\"ucken, Germany \\
\url{muandet@cispa.de}, \url{km@cifer.ai}
}

\date{\today}

\newcommand{\hc}{\ensuremath{H}}
\newcommand{\hsc}{\ensuremath{\mathcal{H}}}
\newcommand{\fc}{\ensuremath{\mathcal{F}}}
\newcommand{\gc}{\ensuremath{\mathcal{G}}}

\newcommand{\hsp}{\ensuremath{\mathscr{H}}}
\newcommand{\hcol}{\ensuremath{B}}
\newcommand{\algo}{\ensuremath{\mathbb{A}}}

\newcommand{\lscol}{\ensuremath{C}}
\newcommand{\msel}{\ensuremath{\mathbb{S}}}
\newcommand{\lcol}[2]{\ensuremath{\text{LS}(#1,#2)}}

\newcommand{\rr}{\ensuremath{\mathbb{R}}}

\newcommand{\X}{\ensuremath{\mathcal{X}}}

\newcommand{\rp}{\ensuremath{\mathbf{r}}}

\begin{document}
\maketitle

\begin{bibunit}[alp]

\begin{abstract}
  Modern applications of AI involve training and deploying machine learning models across heterogeneous and potentially massive environments. 
  Emerging diversity of data not only brings about new possibilities to advance AI systems, but also restricts the extent to which information can be shared across environments due to pressing concerns such as privacy, security, and equity.
  Based on a novel characterization of learning algorithms as choice correspondences on a hypothesis space, this work provides a minimum requirement in terms of intuitive and reasonable axioms under which the only rational learning algorithm in heterogeneous environments is an empirical risk minimization (ERM) that unilaterally learns from a single environment without information sharing across environments.
  Our (im)possibility result underscores the fundamental trade-off that any algorithms will face in order to achieve Collective Intelligence (CI), i.e., the ability to learn across heterogeneous environments.
  Ultimately, collective learning in heterogeneous environments are inherently hard because, in critical areas of machine learning such as out-of-distribution generalization, federated/collaborative learning, algorithmic fairness, and multi-modal learning, it can be infeasible to make meaningful comparisons of model predictive performance across environments.

  \vspace{1em}
\end{abstract}

\section{Introduction}

Artificial intelligence (AI) systems are ubiquitous in every part of society and business. 
The main driving force of its success is a general-purpose method called machine learning (ML) that can turn gigantic amount of data into a powerful predictive model.
Some become recommendation engines \citep{Konstan12:Recommender, Harper15:MovieLens}, some become facial recognition systems \citep{Kamel93:Face, Zhu12:FaceDP, Schroff15:FaceNet, Buolamwini18:GenderShades}, some become large language models \citep{Aharoni19:Massively,Brown20:LLM,OpenAI23:GPT4}, and so on.
Traditionally, training data for ML algorithms are assumed identically and independently distributed (i.i.d) because they often come from a homogeneous environment.
In this scenario, capability to train ML models with billion parameters at scale is the cornerstone of several AI-driven breakthroughs in science \citep{Jumper21:ALphaFold} and engineering \citep{Mirhoseini21:ChipDesign}.
On the other hand, democratization of AI requires these models to be trained and deployed across heterogeneous and potentially massive environments where the i.i.d. assumption is almost always violated. 
For example, multi-source data is an essential part in multi-task learning \citep{Caruana97:MTL,Zhang21:MTL-Review}, domain generalization (DG) \citep{Blanchard11:Generalize, Muandet13:DG, Mahajan21:DG-CausalMatching, Wang21:DG-Review, Zhou21:DG-Review, Singh24:DGIL}, and out-of-distribution (OOD) generalization \citep{Arjovsky19:OOD,Wald21:OOD-Calibration} that improves model performance. 
Also, real-world data often arrive in different modalities ranging from visual data (e.g., images and videos) to language data (e.g., text and speech).
Some have recognized cross-modal learning as a hallmark of artificial general intelligence (AGI) \citep{Jean-Baptiste22:Flamingo,Reed22:Gato}.

But with great power also comes great responsibility. 
Such predictive models have recently been used in aiding high-stakes decision-making in health care \citep{Tomav19:Clinical,Wiens19:Roadmap,Ghassemi22:ML-Health}, employment \citep{Deshpande20:Ai-Resume}, and criminal justice \citep{Angwin16:MachineBias}. 
In human-centric domains, heterogeneity observed across data points, on the one hand, may represent an inherent diversity within a population that the learning algorithm must account for \citep{Heckman01:Heterogeneity}.
On the other hand, it may reflect prejudice and societal biases against specific demographic groups that have historically influenced the collection and distribution of data. 
As a result, a growing concern on the disproportionate impact of algorithmic models has not only sparked a cross-disciplinary collaborations to increase fairness and transparency in today's AI systems, but also created a sense of responsibility among legislatures to regulate them \citep{Hardt16:EO,Barocas19:FML,Zafar19:Fairness,Kilbertus20:FairImperfect,Giffen22:Pitfalls}.
Additionally, when training data are scattered over potentially massive network of remote devices such as mobile phones or siloed data from hospitals, matters pertaining to privacy, security, and access rights may prevent data sharing across environments.
Federated learning (FL), for example, addresses these challenges by developing algorithms that rely only on privacy-aware aggregated information without direct access to local data sets \citep{Konecny16:FedOpt,Mcmahan17:FL,Li20:FL,Kairouz21:FL}.

The evolving diversity within data presents promising prospects for enhancing our existing AI systems, but a critical trade-off inherent in tackling learning problems within such context remains inadequately comprehended, particularly as we grapple with real-world constraints like privacy, security, and equity.

To gain better understanding on this trade-off, this work aims to answer the following question: \emph{Given training data from several heterogeneous environments, is it possible to design a rational  learning algorithm that can learn successfully across these environments?}
Inspired by the Arrow's Impossibility Theorem \citep{Arrow50:Impossibility} and its descendants in social choice theory \citep{Sen70:CCSW,Sen17:CCSW}\footnote{See \citet{Patty19:SocialChoice} for a recent survey.}, we adopted an axiomatic method to systematically answer this question. 
In particular, we provide intuitive and reasonable axioms under which an empirical risk minimization (ERM) is the only rational learning algorithm in heterogeneous environments.\footnote{This work refers to ERM loosely as the algorithm that learns by minimizing the \emph{empirical} risk function using data collected from a single environment (\Cref{sec:homogeneous}). When heterogeneity is disregarded, another popular baseline is ERM on data pooled from all environments as if it was a single i.i.d. data set. However, it can be deemed practically undesirable in many critical applications such as health care and criminal justice.}
Our key insight is that a learning algorithm is nothing, but an implementation of choice behaviour of its designer over a model class, which can then be modelled by a choice correspondence, i.e., a function mapping a set of candidate models to a corresponding subset of optimal solutions (cf. \Cref{def:learning-rule}). 
This formulation allows us to subsequently impose the desirable properties of any conceivable learning algorithms, namely, Pareto Optimality (PO), Independence of Irrelevant Hypotheses (IIH), Invariance Restriction (IR), and Collective Intelligence (CI).
Intuitively, the PO property generalizes the notion of ``minimum risk'' in the classical setting to heterogeneous environments. 
Next, the IIH property demands that learning algorithms should not be sensitive to ``irrelevant'' information when choosing between two models.
Similarly, the IR property ensures that the algorithm's outputs remain unchanged for the risk functionals that are informationally identical.
Last but not least, we argue that the CI property, which demands that the algorithm leverages information across multiple environments, is indispensable for two reasons.
First, an algorithm that lacks CI property fails to recognize the added diversity of data and thereby does not lead to any meaningful improvement over ERM. 
Second, the lack of CI means that one (or a few) of the heterogeneous environments, which may correspond to a specific individual, demographic group, or institution, over-proportionally determines the outcomes of the ML algorithm (see \Cref{sec:axioms} and \Cref{sec:implications} for an in-depth discussion).


\subsection{Our Contributions}

We perceive the challenges of learning in heterogeneous environments as an aggregation function
\begin{equation*}
    F: (r_1,\ldots,r_n) \mapsto (\hc,\hcol,\algo), \quad n \in \mathbb{N},
\end{equation*}
which entails using a risk profile comprising risk functions calculated across $n$ environments to formulate a blueprint of a learning algorithm. This algorithm selectively determines the optimal models from the hypothesis class $\hc$, i.e., $\algo(\hsc)\subseteq\hsc$ for all $\hsc\in\hcol$ where $\hcol$ is a collection of non-empty subsets of $\hc$.
We establish the fundamental properties of both the aggregation function $F$ and the blueprint of learning algorithms $(\hc,\hcol,\algo)$.

The contributions of this work can be summarized as follows:
\begin{itemize}
    \item We introduce a novel choice-theoretic perspective of machine learning. 
    A learning problem is expressed mathematically as a triple $(\hc,\hcol,\algo)$, called a \emph{learning structure}, where $\hc$ is a subset of the hypothesis space $\hsp$ comprising all conceivable hypotheses pertaining to the learning problem. 
    Designed by the researchers, $\hcol$ and $\algo$ denote respectively a collection of non-empty subsets of $\hc$ and a learning rule represented as a choice correspondence that specifies for any
    feasible set in $\hcol$ a nonempty subset of hypotheses (cf. \Cref{def:learning-rule}).  
    The learning structure $(\hc,\hcol,\algo)$ serves as a building block of a model selection structure (cf. \Cref{def:model-selection-structure}) and subsequently the two-stage model of machine learning (cf. \Cref{sec:two-stage-model}).
    Based on this formulation, we provide a characterization of rational learning algorithms (cf. \Cref{sec:internal-consistency}).
    Our characterization provides a more holistic view of learning algorithms, which subsumes the conventional risk-minimization perspective. 

    \item Agnostic to specific applications, we argue that Pareto Optimality (PO), Independence of Irrelevant Hypotheses (IIH), Invariance Restriction (IR), and Collective Intelligence (CI) are primitive properties of any aggregation function $F$. 
    When there is a finite number of two or more environments and at least three distinct hypotheses, we subsequently show that 
    the only rational algorithm that is compatible with PO, IIH, and IR is an empirical risk minimization (ERM) that unilaterally learns from a single environment (cf. \Cref{lem:erm}), contradicting the CI property.
    The implication of this is that there cannot exist an aggregation function that satisfies these properties simultaneously, yet can produce rational learning algorithm (cf. \Cref{thm:impossibility}). 
    In other words, we establish the fundamental trade-off for learning algorithms that can learn successfully across heterogeneous environments.

    \item We provide a thorough discussion on practical implications of our general (im)possibility results in critical areas of modern machine learning including OOD generalization, federated learning, collaborative learning, algorithmic fairness, and multi-modal learning (cf. \Cref{sec:implications}). 
    In each of these areas, we provide concrete examples to illustrate how our results can be practically relevant and also suggest ways to overcoming this trade-off (cf. \Cref{sec:escaping}). 
    Most notably, we argue that learning in heterogeneous environments is inherently hard because, unlike in a homogeneous environment, it may not be feasible to always make meaningful comparisons of model performances across environments due to physical, ethical, and legal constraints. 
    As a result, we advocate for secure and trustworthy mechanisms that enable and incentivize dissemination of relevant information across environments as a prerequisite for generalizable, fair, and democratic learning algorithms.
\end{itemize}

The paper is organized as follows. \Cref{sec:learning-rule} introduces a choice-theoretic perspective of machine learning, followed by a characterization of rational algorithms in \Cref{sec:internal-consistency}. 
Next, \Cref{sec:heterogeneous-learning} provides a formulation of learning problem in heterogeneous environments and presents our (im)possibility results.
\Cref{sec:escaping} discusses ways to overcome this impossibility, followed by direct implications of our main results on critical domains of machine learning in \Cref{sec:implications}. 
Finally, Section \ref{sec:conclusion} discusses the limitations of this work and concludes the paper.

\section{Algorithmic Choice}
\label{sec:learning-rule}

This section introduces a learning structure $(\hc,\hcol,\algo)$ as a blueprint of learning problems.

\subsection{Learning Structure}

Let $\hsp$ be a (possibly infinite) set of all conceivable hypotheses. 
A learning problem involves the process of designing, implementing, and executing a program that chooses the best solutions from $\hsp$.
Specifically, a learning algorithm $\algo$ specifies for any feasible non-empty subset $\hc\subseteq\hsp$ a nonempty subset $\algo(\hc)\subseteq\hc$. 
We refer to each subset $\hc$ of $\hsp$ as a model class.
The set of natural numbers is denoted by $\mathbb{N}$, and for $n\in\mathbb{N}$, $[n]:=\{1,2,\ldots,n\}$. The set of real numbers and its positive part are $\rr$ and $\rr_+$, respectively.

\begin{definition}
\label{def:learning-rule}
A learning structure is a triple $(\hc,\hcol,\algo)$ where $\hc$ is a non-empty subset of $\hsp$ that consists of all feasible hypotheses and $\hcol$ is a collection of non-empty subsets of $\hc$. A learning algorithm is a mapping $\algo:\hcol\to2^{\hc}$ that fulfills the following conditions:
\begin{enumerate}[label=(\roman*)]
    \item \textbf{Nonresponsiveness}: $\algo(\emptyset) = \emptyset$.
    \item \textbf{Properness}: $\algo(\hsc) \subseteq \hsc$ for all $\hsc\in\hcol$.
    \item \textbf{Conclusiveness}: $\algo(\hsc) \neq \emptyset$ for all $\hsc\in\hcol$.
\end{enumerate}
\end{definition}

That is, we view a learning algorithm $\algo$ as a choice correspondence defined over $\hcol$ (see Figure \ref{fig:learning-structure}). 
Firstly, nonresponsiveness requires that the algorithm does not hallucinate the solutions out of thin air.\footnote{Since $\hcol$ consists of only non-empty subsets of $\hc$, the nonresponsiveness condition can be omitted without loss of generality.}
Secondly, the properness condition presupposes that the solutions in $\algo(\hsc)$ belong to the hypothesis class $\hsc$ that was originally given to the algorithm.
We call $\algo$ \emph{proper} (resp. \emph{improper}) algorithm if it satisfies (resp. violates) this condition; see also \citet[Remark 3.2]{Shalev-Shwartz14:UML}. 
Further restriction can be made so that any $\algo(\hsc)$ must be a unit set, with only one hypothesis chosen from $\hsc$, but we stick to the more general setting throughout.
Lastly, the conclusiveness condition requires that $\algo$ cannot remain inconclusive in the sense that the algorithm must produce something when learning on a valid hypothesis class.
Nevertheless, it is not impossible to deliberately construct a learning algorithm where $\mathbb{A}(\mathcal{H}) = \emptyset$ for some $\mathcal{H}$. This requirement can be relaxed without altering our main results, though at the cost of additional technical challenges.

One can broadly understand the set $\algo(\hsc)$ as a subset of solutions that can potentially be chosen from $\hsc$. 
An arbitrary hypothesis $f$ is considered \emph{as good as} other hypotheses in $\hsc$ if $f\in\algo(\hsc)$. We say that $f$ is \emph{better than} $g$ with respect to $\hsc$ if $f\in\algo(\hsc)$, but $g\notin\algo(\hsc)$. If $f,g\in\algo(\hsc)$, 
then they are both considered as \emph{equally good} by the algorithm. 
For example, the most ubiquitous algorithm in machine learning is a risk minimizer which can be expressed mathematically as $\algo(\hsc) = \{h\in\hsc\,:\, r(h) \leq r(g), \forall g\in\hsc\}$ for $\hsc\in\hcol$ and some real-valued risk functional $r: \hc\to\rr$. In this case, $f$ is at least as good as $g$ if $r(f) \leq r(g)$ and is better than $g$ if $r(f) < r(g)$. They are equally good if $r(f) = r(g)$. 
Nevertheless, by abstracting away the internal learning procedure, the learning structure encompasses more than minimising risk function; it can entail more involved learning procedures that are prevalent in heterogeneous environments such as aggregating, open-sourcing, and trading algorithmic models, for example.
We provide further discussion in \Cref{sec:homogeneous} and \Cref{sec:risk-minimizers}.

To summarize, we view learning as a choice problem on a hypothesis space $\hsp$ of which traditional risk minimization is a special case (see \cref{rem:risk-minimizers}).
Given possibly infinite choices of solutions in $\hsp$, the learning structure $(\hc,\hcol,\algo)$ describes a specific way in which the solutions could possibly be learned from $\hsp$. 
The model class $\hc$ consists of feasible hypotheses generated by a specific model, e.g., deep neural networks or kernel machines, whereas the algorithm $\algo$ corresponds to a learning procedure for all non-empty subsets in $\hcol$, e.g., stochastic gradient descent (SGD) or convex optimization methods.
The collection $\hcol$, which depends jointly on $\hc$ and $\algo$, should be thought of as an exhaustive listing of all the learning problems that can conceivably be posed to the learning algorithm $\algo$.
The Axiom of Choice (AC) allows us to reason about existence of $\algo$ for any $\hcol$ \citep{Zermelo04:AC,Moore82:AC}.
Therefore, while $\hsp$ is determined by the learning problem at hand, $(\hc,\hcol,\algo)$ depends entirely on how the ML researchers would solve it.

\begin{figure}[t!]
    \centering
    \resizebox{!}{0.35\textwidth}{
    \begin{tikzpicture}
        \coordinate (F) at (5,0);

        \draw [fill = gray, line width=1.5pt, opacity=0.1] (5,0) ellipse (3.5cm and 2.5cm);

        \draw [ultra thick] (5,0) ellipse (3.5cm and 2.5cm) node[above, yshift=1.8cm, xshift=2.7cm] {$\hsp$};

        \shade[ball color = yellow, opacity = 0.2] (5,0) circle [radius=2cm];
        \shade[ball color = green, opacity = 0.3] (5,-0.7) circle [radius=1cm]node[above,opacity=1,xshift=-0.4cm,yshift=-0.2cm] {$\hsc_3$};;
        \shade[ball color = green, opacity = 0.3] (5.7,0) circle [radius=0.9cm]node[above,opacity=1,xshift=0.4cm] {$\hsc_2$};
        \shade[ball color = green, opacity = 0.3] (4.5,0.5) circle [radius=0.9cm]node[above,opacity=1,xshift=-0.2cm] {$\hsc_1$};

        \draw (F) circle [very thick, radius=2cm] node[above, yshift=1.8cm, xshift=1cm] {$\hc$};
        
        \shade[ball color = red, opacity = 0.3] (5,-1.2) circle [radius=0.4cm]node[opacity=1] {$h^*$};

        \path [draw,ultra thick] (-1.8,-1.7) rectangle (-0.2,-0.1);
        \node (S) at (-1,-0.9) {\Large $\algo$};
        \draw [dashed,->,line width=1pt] (4,-0.5) |- (-0.1,-0.5);
        \draw [dashed,->,line width=1pt] (S)++(0.9,-0.3) |- (4.6,-1.2);
        
        \node at (0,2) {$\hcol=\{\hsc_1,\hsc_2,\hsc_3\}$};
        \node at (0,1.2) {$\algo(\hsc_3)=\{h^*\}$};
        
    \end{tikzpicture}}
    \caption{A learning structure $(\hc,\hcol,\algo)$. For each $\hsc\in\hcol$, the set $\algo(\hsc)$ consists of the optimal hypotheses $h^*$ learned by the algorithm $\algo$. The hypothesis space $\hsp$ is determined by the learning problem at hand, whereas the learning structure $(\hc,\hcol,\algo)$ is a design choice.}
    \label{fig:learning-structure}
\end{figure}


\subsection{Model Selection Structure}
\label{sec:model-selection-structure}

Machine learning relies extensively on a model selection strategy whereby the ML engineers decide based on domain knowledge or historical data which model class, learning algorithm, and hyperparameter values are most suitable for the problem at hand. 
In other words, it involves choosing from a space of all possible configurations of model classes and learning algorithms. 
Hence, it is a \emph{learning problem over learning structures} as we formalize below.

Let $\Omega$ and $\Xi$ be arbitrary (countable or uncountable) index sets. We formally define a collection of model classes as $\{\hc_\omega\,:\,\omega\in\Omega\}$ where $\hc_\omega\subset\hsp$ are the model classes indexed by $\omega\in\Omega$ and a collection of learning algorithms as $\{\algo_\xi\,:\, \xi\in\Xi\}$ where $\algo_\xi$ are the learning algorithms indexed by $\xi\in\Xi$.
Subsequently, we define a collection of learning structures with respect to the hypothesis space $\hsp$ and the index sets $\Omega,\Xi$ as
\begin{equation}
    \label{eq:ls-collection}
    \lcol{\Omega}{\Xi} := \{(\hc_\omega,\hcol,\algo_\xi)\,:\, \omega\in\Omega, \xi\in\Xi\}.
\end{equation}
To avoid cluttered notation, we simply write $B$ in \eqref{eq:ls-collection} although it depends on both $\omega$ and $\xi$.
A model selection structure can then be defined as a learning structure over a collection of more elementary learning structures.

\begin{definition}
    \label{def:model-selection-structure}
    A model selection structure is a triple $(\lcol{\Omega}{\Xi},\lscol,\msel)$ where $\lcol{\Omega}{\Xi}$ is a collection of learning structures \eqref{eq:ls-collection} indexed by arbitrary index sets $\Omega,\Xi$ and $\lscol$ is a collection of non-empty subsets of $\lcol{\Omega}{\Xi}$. A model selection strategy is a mapping $\msel:\lscol \to \lcol{\Omega}{\Xi}$ that fulfills the nonresponsiveness, properness, and conclusiveness properties.
\end{definition}

The interpretation of \Cref{def:model-selection-structure} is analogous to that of  \Cref{def:learning-rule}, except that the model selection strategy is defined as a choice function, i.e., any $\msel(\mathcal{E})$ where $\mathcal{E}\in C$ must be a unit set, with only one learning structure chosen from $\text{LS}(\Omega,\Xi)$. That is, model selection is itself a learning problem in which the hypothesis class consists of primitive learning structures.

We interpret the index sets $\Omega$ and $\Xi$ in the broadest sense as sets of hyperparameters associated with the model class $\hc$ and the learning algorithm $\algo$ that generate all conceivable learning structures $\lcol{\Omega}{\Xi}$. 
The collection $\lscol$, which depends jointly on $\lcol{\Omega}{\Xi}$ and $\msel$, should be thought of as an exhaustive listing of all the model selection problems that can conceivably be posed to the model selection strategy $\msel$. 
For instance, each set in $\lscol$ may consists of learning structures that lie within a specific range of hyperparameter values that the domain experts can possibly come up with; see \Cref{sec:homogeneous} for more concrete examples. 

\subsection{A Two-Stage Model of Machine Learning}
\label{sec:two-stage-model}

Both learning structure and model selection structure introduced in the previous two sections constitute a positive model of a typical machine learning pipeline.
The entire pipeline can be summarized in the following diagram:
\begin{equation}\label{eq:two-stage-model}
\hsp \to 
\rlap{$\overbrace{(\lcol{\Omega}{\Xi},\lscol,\msel) \to \msel(\mathcal{E}), \mathcal{E}\in\lscol \to (\hc_{\omega^*},\hcol,\algo_{\xi^*})}^{\text{{Stage I: Model Selection}}}$}
(\lcol{\Omega}{\Xi},\lscol,\msel) \to \msel(\mathcal{E}), \mathcal{E}\in\lscol \to 
\underbrace{(\hc_{\omega^*},\hcol,\algo_{\xi^*}) \to \algo_{\xi^*}(\hsc), \hsc\in\hc_{\omega^*}}_{\text{Stage II: Model Training}}.
\end{equation}
Starting with the hypothesis space $\hsp$, the first stage involves creating a blueprint of model class $\hc$ and learning algorithm $\algo$, giving rise to a collection of learning structures $\lcol{\Omega}{\Xi}$ from which the most desirable one is chosen, i.e., by choosing the best hyperparameter values $\omega^*$ and $\xi^*$ via a model selection procedure $\msel$.
The second stage then involves implementing the algorithm $\algo_{\xi^*}$ to choose the best hypotheses from $\hc_{\omega^*}$, e.g., by training a deep neural network until convergence. 
This two-stage model provides a succinct description for most of the existing machine learning algorithms however complex they are. 
\Cref{fig:ml-pipeline} illustrates this two-stage model.\footnote{To put it in modern machine learning parlance, model selection in \eqref{eq:two-stage-model} can be replaced with \emph{model pre-training} while model training in \eqref{eq:two-stage-model} can be replaced with \emph{model finetuning}.}

The strategy $\msel$ subsumes both the situations in which one directly chooses from $\lcol{\Omega}{\Xi}$ a specific learning structure based on their prior knowledge and those in which sophisticated model selection procedures are used.
In the latter situations, the strategy $\msel$ often involves a repeated execution of the algorithm $\algo_\xi$ on $\hc_\omega$ and then evaluating the chosen hypotheses based on some real-valued score functions such as, among others, cross-validation (CV) error \citep[Ch. 7]{Hastie09:ESL}, Akaike information criterion (AIC) \citep{Akaike74:AIC}, and Bayesian information criterion (BIC) \citep{Schwarz78:BIC}.
That is, the quality of the learning structure $(\hc_{\omega},\hcol,\algo_{\xi})$ is determined by the quality of its outcome, i.e., the chosen hypotheses. 
Since the learning structure $(\hc_{\omega},\hcol,\algo_{\xi})$ lies at the heart of both stages of machine learning pipeline, we focus on characterizing its properties in  \Cref{sec:internal-consistency}.
 
\begin{remark}\label{rem:context-dependence}
    The proposed two-stage model underlines the context-dependent nature of machine learning algorithms. 
    The hypothesis class $\hc_\omega$ can be viewed as the context in which the algorithm $\algo_\xi$ operates. This model allows learning algorithms $\algo$ to behave differently in different contexts depending on the choice of $\xi$.
    Interestingly, this has a close connection to a dynamic choice problem \citep{Kreps79:DyChoice} and context-dependent learning \citep{Tversky93:Context-Pref,Pfannschmidt22:LearningChoice} in mainstream economics. 
    It involves one's choices that are spread out over time.
\end{remark}

\begin{figure}[t!]
    \centering
    \resizebox{0.9\textwidth}{!}{
    \begin{tikzpicture}
        \coordinate (F) at (1,0);
        \coordinate (S) at (-7.8,0.5);
        \coordinate (G) at (-11.5,0);
        \coordinate (H) at (8,0);
        \coordinate (L) at (-9,0);
        \coordinate (W) at (-10,-2);

        \shade[ball color = green, opacity = 0.4] (1,0) circle [radius=1.5cm];
        \shade[ball color = green, opacity = 0.2] (-7.8,0.5) circle [radius=2.5cm];
        \shade[ball color = red, opacity = 0.2] (-11.5,0) circle [radius=2cm];
        \shade[ball color = yellow, opacity = 0.2] (-10,-2) circle [radius=1cm];

        \draw [fill = gray, line width=1.5pt, opacity=0.1] (-9,0) ellipse (5cm and 3.5cm);
        \draw [ultra thick] (-9,0) ellipse (5cm and 3.5cm) node[above, yshift=3.5cm, xshift=0cm] {\Large $\lcol{\Omega}{\Xi}$};
        
        \draw [dashed, very thick] (S) circle [line width=4pt, radius=2.5cm] node[above, xshift=-2cm, yshift=2cm] {\Large $\mathcal{E}$};
        \draw [thick] (F) circle [thick, radius=1.5cm] node[above, yshift=1.5cm] {\Large $\hc_{\omega_{\diamond}}$};
        \draw [dashed, very thick] (G) circle [thick, radius=2cm] node[above, yshift=2.2cm] {};
        \draw [dashed, very thick] (W) circle [radius=1cm] node[above, yshift=2.2cm] {};
        
        \shade[ball color = red, opacity = 0.4] (G) circle [radius=1cm, xshift=7mm];
        \draw (G) [thick] circle [radius=1cm, xshift=7mm];
        \shade[ball color = red, opacity = 0.4] (G) circle [radius=1cm, xshift=-5mm, yshift=5mm];
        \draw (G) [thick] circle [radius=1cm, xshift=-5mm, yshift=5mm];
        \shade[ball color = red, opacity = 0.4] (G) circle [radius=1cm, xshift=-5mm, yshift=-5mm];
        \draw (G) [thick] circle [radius=1cm, xshift=-5mm, yshift=-5mm];

        \shade[ball color = green, opacity = 0.5] (F) circle [radius=0.9cm, xshift=3mm, yshift=-2mm];
        \shade[ball color = green, opacity = 0.5] (F) circle [radius=0.7cm, xshift=-4mm, yshift=4mm];
        \node at (F)[xshift=3mm, yshift=-6mm] {$\bigstar$};
        \node at (F)[xshift=-4mm, yshift=7mm] {$\bigstar$};

        \shade[ball color = green, opacity = 0.4] (S) circle [radius=0.8cm, xshift=-2mm, yshift=-15mm];
        \draw (S) [thick] circle [radius=0.8cm, xshift=-2mm, yshift=-15mm];
        \shade[ball color = green, opacity = 0.4] (S) circle [radius=1.1cm, xshift=-11mm, yshift=-1mm];
        \draw (S) [thick] circle [radius=1.1cm, xshift=-11mm, yshift=-1mm];
        \shade[ball color = green, opacity = 0.4] (S) circle [radius=1.6cm, xshift=3mm, yshift=2mm];
        \draw (S) [thick] circle [radius=1.6cm, xshift=3mm, yshift=2mm];

        \shade[ball color = green, opacity = 0.5] (S) circle [radius=0.7cm, xshift=-1mm, yshift=7mm];
        \shade[ball color = green, opacity = 0.5] (S) circle [radius=1cm, xshift=5mm, yshift=0mm];
        \node at (S)[xshift=-1mm, yshift=10mm] {$\bigstar$}; 
        \node at (S)[xshift=5mm, yshift=-5mm] {$\bigstar$}; 

        \shade[ball color = yellow, opacity = 0.4] (W) circle [radius=0.5cm, xshift=3mm];
        \draw (W) [thick] circle [radius=0.5cm, xshift=3mm];
        \shade[ball color = yellow, opacity = 0.4] (W) circle [radius=0.5cm, xshift=-3mm, yshift=3mm];
        \draw (W) [thick] circle [radius=0.5cm, xshift=-3mm, yshift=3mm];
        \shade[ball color = yellow, opacity = 0.4] (W) circle [radius=0.5cm, xshift=-3mm, yshift=-3mm];
        \draw (W) [thick] circle [radius=0.5cm, xshift=-3mm, yshift=-3mm];

       \node at (-5,6) {\Large\textbf{Stage I: Model Selection}}; 
       \node at (3.5,6) {\Large\textbf{Stage II: Model Training}}; 

       \node at (-5,5.3) {\Large $(\lcol{\Omega}{\Xi},\lscol,\msel)$}; 
       \node at (3.5,5.3) {\Large $(\hc_{\omega_{\diamond}},\hcol,\algo_{\xi_{\diamond}})$};

       \node at (6.5,0) {\Huge $\bigstar$};
       \node at (-8.5,2.5) {\Large $\hc_{\omega}$}; 
       \node at (1.7,0.9) {$\hsc$}; 
        
        \draw[->,ultra thick] (-4,0) -- node[below,yshift=-0.5em] {\Large $\msel(\mathcal{E}),\mathcal{E}\in\lscol$} (-0.5,0);
        \draw[->,ultra thick] (2.5,0) -- node[below,yshift=-0.5em] {\Large $\algo_{\xi_{\diamond}}(\hsc), \hsc\in \hcol$} (6,0);
        
    \end{tikzpicture}}
    \caption{\textbf{A two-stage model of machine learning}: We model a typical machine learning pipeline as a two-stage choice. Given a hypothesis space $\hsp$, one must first come up with a \emph{model} and a \emph{learning algorithm}. 
    The model induces a collection of hypothesis class $\hc_{\omega}$ parametrized by some hyperparameter $\omega\in\Omega$, while the learning algorithm $\algo_{\xi}$ is a choice correspondence parametrized by $\xi\in\Xi$. The algorithm $\algo_\xi$ prescribes a series of instructions that will be executed to choose the best solutions from the hypothesis class $\hc_{\omega}$ or any subsets thereof. 
    The result of this design process is a \emph{learning structure} $(\hc_{\omega},\hcol,\algo_{\xi})$ where $\hcol$ is a collection of nonempty subsets of $\hc_{\omega}$ for which $\algo_{\xi}(\hsc)\subseteq \hsc$ and $\algo_{\xi}(\hsc)\neq\emptyset$ for all $\hsc\in\hcol$. 
    A model selection structure $(\lcol{\Omega}{\Xi},\lscol,\msel)$ consists of a collection of learning structure (leftmost solid oval), a collection of subsets of $\lcol{\Omega}{\Xi}$ (dashed circles), and a model selection procedure $\msel$.
    In Stage I, the model selection $\msel$ chooses the learning structure $(\hc_{\omega_{\diamond}},\hcol,\algo_{\xi_{\diamond}})$ from a subset $\mathcal{E}$ in $\lscol$ (green solid circle).
    For example, one may choose the best learning structure by either manually setting the values of the hyperparameters (i.e., $\lscol$ is a collection of singletons) or by adopting data-dependent model selection procedures (i.e., $\lscol$ is composed of nontrivial subsets). 
    In Stage II, the choice is delegated subsequently to the learning algorithm $\algo_{\xi_{\diamond}}$ which chooses from the model class $\hc_{\omega}$ or any subsets thereof. The final outcome of this process is a set of optimal solutions denoted in the figure by $\bigstar$.}
    \label{fig:ml-pipeline}
\end{figure}

\subsection{Empirical Risk Minimization}
\label{sec:homogeneous}

In supervised learning, the most influential learning algorithm is an empirical risk minimization (ERM) \citep{Vapnik91:ERM}. 
Let $\mathcal{X}$ and $\mathcal{Y}$ be non-empty input and output spaces, respectively, $X$ and $Y$ be random variables taking values in $\mathcal{X}$ and $\mathcal{Y}$.
Furthermore, the input space $\mathcal{X}$ is often assumed a Polish space, i.e., a separable and completely metrizable topological space, whereas the output space $\mathcal{Y}$ is a subset of $\rr$.
We denote the realizations of $X$ (resp. $Y$) by $x$ (resp. $y$). 
Given a loss function $\ell: \mathcal{Y}\times\mathcal{Y}\to\mathbb{R}_+$, supervised learning aims to find $f:\mathcal{X}\to\mathcal{Y}$ that minimizes the expected loss 
\begin{equation}\label{eq:expected-risk}
    r(f) = \int_{\mathcal{X}\times\mathcal{Y}} \ell(y,f(x))\,dP(x,y)
\end{equation}
computed with respect to some fixed, but unknown probability distribution $P(X,Y)$ on $\mathcal{X}\times\mathcal{Y}$. 
For example, in animal species categorization, the input $x$ is a photo from a camera trap and the label $y$ corresponds to one of the animal species \citep{Beery18:WildLife,Beery20:iWildCam}. 
Accurate species classification from camera traps can help ecologists better understand wildlife biodiversity and monitor endangered species.

In practice, we only observe an independent and identically distributed (i.i.d.) sample $(x_i,y_i)_{i=1}^m$ of size $m$ from $P(X,Y)$.
The ERM framework, which is the backbone of supervised learning, approximates the expected loss \eqref{eq:expected-risk} by a \emph{regularized} empirical counterpart 
\begin{equation}\label{eq:erm-pre}
    \hat{r}(h) = \frac{1}{m}\sum_{i=1}^m\ell(y_i,h(x_i)) + \lambda\Lambda(\|h\|_{\hc})
\end{equation}
where $\Lambda(\cdot)$ is a monotonically increasing function, $\lambda$ is a positive real-valued regularization constant, and $\hc$ is a hypothesis class.
The second term on the right-hand side of \eqref{eq:erm-pre} is a regularization term introduced to avoid overfitting.
If $\lambda$ decays at the right rate with the sample size $m$, it follows from the law of large number (LLN) that, for each $h\in\hc$, $\hat{r}(h)$ converges in probability to $r(h)$. 
Hence, ERM chooses the best hypotheses from $\hc$ by minimizing the empirical risk \eqref{eq:erm-pre}.  
The generalization capability of the minimizer of \eqref{eq:erm-pre} to unseen data (e.g., unseen camera trap photos) has been studied extensively under the i.i.d. assumption and has been shown to be governed by the uniform consistency of $\hat{r}(h)$ over $\hc$ \citep{Vapnik98:SLT,Cucker02:SLT}.

The learning structures $\{(\hc_\omega,\hcol,\algo_\xi)\,:\, \omega\in\Omega, \xi\in\Xi\}$ associated with the above supervised learning problem can be expressed as follows. 
Firstly, the hypothesis space $\hsp$ consists of all functions from $\mathcal{X}$ to $\mathcal{Y}$. Another common hypothesis space is the space of real-valued square-integrable functions on $\mathcal{X}$ with respect to a measure $\mu$, i.e., $L^2(\mathcal{X},\mu)$.
The index set $\Omega$ consists mainly of different choices of the norm $\|\cdot\|_{\hc_{\omega}}$ used as a regularizer in \eqref{eq:erm-pre} and $\hc_{\omega}$ consists of functions $h$ in $\hsp$ whose norm $\|h\|_{\hc_{\omega}}$ is well-defined, whereas $\Xi$ consists of the choices of $\xi = (\ell, \lambda, \Lambda)$, and other relevant hyperparameters such as optimizers and learning rate. 
Lastly, the learning algorithm can be written as 
$\algo_{\xi}(\hsc) = \{h \in \hsc\,:\, \hat{r}(h) \leq \hat{r}(g), \forall g\in\hsc \}$ for $\hsc\in\hcol$ where $\hcol$ is a collection of non-empty subsets of $\hc_{\omega}$.
The ERM belongs to a class of algorithms known as risk minimizers (RM); see \Cref{rem:risk-minimizers}.

In what follows, we provide more concrete examples of various learning structures.

\begin{example}[Kernel machines]
Kernel machines such as support vector machine (SVM) and kernel ridge regression (KRR) learn by choosing solutions from a reproducing kernel Hilbert space (RKHS) \citep{aronszajn50reproducing,Scholkopf01:LKS}. 
Suppose that $\mathcal{Y}\subseteq\rr$, the model class $\hc_{\omega}$ is a Hilbert space consisting of real-valued functions on $\mathcal{X}$ with $\langle \cdot,\cdot\rangle_{\hc_{\omega}}$ and $\|\cdot\|_{\hc_{\omega}}$ being its inner-product and norm, respectively. It is called a reproducing kernel Hilbert space (RKHS) if there exists a symmetric function $k:\mathcal{X}\times\mathcal{X}\to\rr$, called a reproducing kernel of $\hc_{\omega}$, satisfying the following properties:
\begin{enumerate}
    \item For all $x\in\mathcal{X}$, we have $k(x,\cdot)\in\hc_{\omega}$ where $k(x,\cdot)$ is the function of the second argument with $x$ being fixed such that $x' \mapsto k(x,x')$.
    \item For all $f\in\hc_{\omega}$ and $x\in\mathcal{X}$, we have $f(x) = \langle k(x,\cdot), f \rangle_{\hc_{\omega}}$. This is known as the reproducing property of $\hc_{\omega}$.
\end{enumerate}
Since $\hc_{\omega}$ is uniquely determined by its reproducing kernel $k$ \citep{aronszajn50reproducing}, the choice of $\hc_{\omega}$ boils down to the choice of $k$.
Popular kernels on $\mathcal{X}\subseteq\rr^d$ include linear kernels $k(x,x')=x^{\top}x'$, polynomial kernels $k(x,x') = (x^{\top}x' + c)^p, c > 0, p\in\mathbb{N}_+$, Gaussian kernels $k(x,x') = \exp(-\|x-x'\|_2^2/2\sigma^2), \sigma > 0$, and Laplace (or more generally Mat\'{e}rn) kernels $k(x,x') = \exp(-\|x-x'\|_2/2\sigma^2), \sigma > 0$. More details on kernel machines can be found in \citet{Scholkopf01:LKS}, \citet{HofSchSmo08}, and \citet{Muandet17:KME}, for example.

For kernel machines, the index set $\Omega$ comprises the choice of kernel functions $k$ and their associated hyperparameters, e.g., $c$, $p$, and $\sigma$, which are typically chosen by the model selection procedure (Stage I in \eqref{eq:two-stage-model} and \Cref{fig:ml-pipeline}).
In Stage II, we can express the model training mathematically as 
\begin{equation}\label{eq:kernel-machines}
    \algo_{\xi}(\hsc) = \{h\in\hsc \,:\, \hat{r}(h) \leq \hat{r}(g), \forall g\in\hsc\}
\end{equation}
for $\hsc\subseteq\hc_{\omega}$ where $$\hat{r}(f) = \frac{1}{m}\sum_{i=1}^m\ell(y_i,f(x_i)) + \lambda\|f\|^2_{\hc_{\omega}}.$$
That is, the learning algorithm seeks hypotheses in $\hsc$ that attain the minimum risk.
The index set $\Xi$ thus comprises the choice of loss functions $\ell$, regularization parameter $\lambda$, optimizers, and random seeds among others.
For instance, the loss function for SVM is the hinge loss $\ell(y,y') = \max(0,1-y\cdot y')$ with $y,y'\in\{-1,+1\}$, whereas the square loss $\ell(y,y') = (y-y')^2$ where $y,y'\in\rr$ is commonly used in KRR. 
\end{example}

\begin{example}[Deep learning]
Deep learning has gained recent popularity due to its effectiveness in practice \citep{Lecun15:DL,Goodfellow16:DL}. 
In deep learning, the model class $\hc_{\omega}$ is parameterized by a feedforward neural network (or multilayer perceptrons) and is composed of functions 
\begin{equation}\label{eq:feedforward}
    f(x) := g\circ f_k\circ\cdots\circ~f_2\circ~f_1(x).
\end{equation}
Each $f_i$ is a composed multivariate function:
\begin{equation*}
    f_i(z) := \sigma_i\left(W_i^\top z + \bm{b}_i\right)
\end{equation*}
where $\sigma_i:\rr\to\rr$ are real-valued activation functions, $W_i$ are learnable weights, and $\bm{b}_i$ are bias terms. 
The function $g$ is the last-layer transformation such as identity or sigmoid functions.
There exist many choices for the activation function, but the commonly used ones are sigmoid $\sigma(a) = 1/(1+e^{-a})$, tanh $\sigma(a) = (e^a - e^{-a})/(e^a + e^{-a})$, and rectified linear unit (ReLU) $\sigma(a) = \max\{0,a\}$, to name a few. 
Alternatively, one may express \eqref{eq:feedforward} as $f(x) = g(\bm{w}^\top\phi(x;\theta))$ where $\phi(x)$ denotes a feature embedding of $x$ parametrized by a parameter vector $\theta$ and $\bm{w}$ is a learnable parameter vector.

Early work in deep learning heavily aim at incorporating inductive biases through architectural designs of neural networks, resulting in more specialized architectures (and learning algorithms) such as convolutional neural networks (CNN) \citep{Krizhevsky12:ImageNet,Cohen16:Group}, recurrent neural networks (RNN) with feedback connections \citep{Elman90:RNN}, Transformers with attention mechanism \citep{Vaswani17:Attention}, and graph neural networks (GNN) \citep{battaglia2018relational}.
Unlike kernel machines, the choice of $\hc_\omega$ in deep learning is determined by the choice of network architecture.
Hence, the index $\omega$ specifies the configuration of the neural network, e.g., architecture, number of hidden layers, activation functions, etc. 
The most popular DL algorithm is a stochastic gradient descent (SGD) whose aim is to minimize a variant of \eqref{eq:erm-pre} efficiently. 
The index $\xi$ of the learning algorithm $\algo_{\xi}$ therefore consists of hyperparameters such as learning rate, minibatch size, and the number of epoches, regularization scheme, optimizers, initialization, etc.
\end{example}

\subsection{Summary}
\label{sec:algorithmic-choice-discussion}

To summarize, we view learning as a choice problem over a hypothesis space $\hsp$ and introduce the learning structure $(\hc,\hcol,\algo)$ as a blueprint of a learning problem, consisting mainly of a hypothesis class $\hc$ and learning algorithm $\algo$. We then use it to construct the model selection structure $(\lcol{\Omega}{\Xi},\lscol,\msel)$ where $\lcol{\Omega}{\Xi}$ is a collection of primitive learning structures $(\hc_\omega,\hcol,\algo_\xi)$, indexed by $\omega\in\Omega$ and $\xi\in\Xi$, from which the model selection strategy $\msel$ chooses the learning structure. 
Both structures constitute the so-called \emph{two-stage model} of machine learning where the first stage involves choosing the best model class and learning algorithm (e.g., via hyper-parameters tuning), while the second stage involves choosing the best models from the model class using the chosen learning algorithm (e.g., via model training). This provides a holistic view of machine learning as a two-stage context-dependent choice problem \citep{Kreps79:DyChoice,Tversky93:Context-Pref,Pfannschmidt22:LearningChoice}. 


\section{Rational Algorithms}
\label{sec:internal-consistency}

In this section, we impose further restriction known as an \emph{internal consistency} on the learning structure introduced in \Cref{def:learning-rule}.

\begin{definition}[Internal consistency]\label{asmp:coherence} 
    Let $\{(\hc_{\omega},\hcol,\algo_{\xi})\,:\, \omega\in\Omega, \xi\in\Xi\}$ be a collection of learning structures for some index sets $\Omega$ and $\Xi$. Then, it is internally consistent if the following conditions hold true for all $\omega\in\Omega$ and $\xi\in\Xi$: For $\fc,\gc\in\hcol$, 
    \begin{enumerate}[label=(\arabic{enumi})]
        \item \label{itm:alpha} if $h\in \algo_{\xi}(\fc)$ and $h\in\gc\subseteq\fc$, then $h\in\algo_{\xi}(\gc)$; \hfill (Property $\alpha$)
        \item \label{itm:beta} if $h\in\gc$ and, for $\gc\subseteq\fc$, $h\in\algo_{\xi}(\fc)$, then $\algo_{\xi}(\gc) \subseteq \algo_{\xi}(\fc)$. \hfill (Property $\beta$)
    \end{enumerate}
\end{definition}

\begin{figure}[t!]
    \centering
    \begin{subfigure}[c]{0.4\textwidth}
    \centering
    \resizebox{!}{0.9\textwidth}{
    \begin{tikzpicture}
        \coordinate (S) at (-5,0);

        \shade[ball color = yellow, opacity = 0.2] (-5,0) circle [radius=2cm];
        \shade[ball color = green, opacity = 0.3] (-4.5,-0.5) circle [radius=1.2cm];

        \draw (S) circle [line width=4pt, radius=2cm] node[above, yshift=2.2cm] {$\fc$};
        \draw (-4.5,-0.5) circle [very thick, radius=1.2cm] node[above, yshift=1.2cm] {$\gc$};
        
        \node at (-4.7,0) {\Large\faStar};
        \node at (-4.2,-1) {\Large\faCircle};
        
        \node at (-5,-2.5) {$\algo_\xi(\fc)=\{\text{\faStar}\}, \; \algo_\xi(\gc)=\{\text{\faCircle}\}$};
        
        \draw [<-,line width=2pt,gray] (-5.3,0.4) -- (-6.2,1.6);
        
    \end{tikzpicture}}
    \caption{Violation of Property $\alpha$}
    \label{subfig:alpha-violation}
    \end{subfigure}
    \hspace{3em}
    \begin{subfigure}[c]{0.4\textwidth}
    \centering
    \resizebox{!}{0.9\textwidth}{
    \begin{tikzpicture}
        \coordinate (T) at (0,0);

        \shade[ball color = yellow, opacity = 0.2] (0,0) circle [radius=2cm];
        
        \shade[ball color = green, opacity = 0.3] (0.5,-0.5) circle [radius=1.2cm];

        \draw (T) circle [very thick, radius=2cm]node[above, yshift=2.2cm] {$\fc$};
        
        \draw (0.5,-0.5) circle [very thick, radius=1.2cm]node[above, yshift=1.2cm] {$\gc$};

        \node at (0.5,0) {\Large\faStar};
        \node at (0.8,-1) {\Large\faCircle};
        \node at (-1.2,0) {\Large\faSquare};

        \node at (0,-2.5) {$\algo_\xi(\gc)=\{\text{\faStar,\,\faCircle}\}, \; \algo_\xi(\fc)=\{\text{\faStar,\,\faSquare}\}$};        
        
        \draw [->,line width=2pt,gray] (-0.3,0.4) -- (-1.2,1.6);
        
    \end{tikzpicture}}
    \caption{Violation of Property $\beta$}
    \label{subfig:beta-violation}
    \end{subfigure}
    \caption{(\subref{subfig:alpha-violation}) Property $\alpha$ is violated when a contraction of hypothesis class, e.g., by fixing the values of some parameters of the model with the same hyperparameters, can change the behaviour of $\algo_\xi$. Here, $\algo_{\xi}$ chooses \faStar~from $\fc$, but \faCircle~from $\gc \subset \fc$ although it still contains \faStar. (\subref{subfig:beta-violation}) Property $\beta$ is violated when an expansion of the hypothesis class, e.g., by adding more parameters of the model with the same hyperparameters, can change how the optimal hypotheses are chosen by $\algo_{\xi}$. Here, $\algo_{\xi}$ chooses \{\faStar, \faCircle\} from $\gc$, but neglects \faCircle~from its optimal choice when choosing from $\fc \supset \gc$.}
    \label{fig:internal-consistency}
\end{figure}

These two conditions, known as Property $\alpha$ \citep{Chernoff54:Alpha}\citep[Ch. $1^*$]{Sen70:CCSW} and Property $\beta$ \citep[pp. 320]{Sen17:CCSW}, together form one of the interpretations of rationality of choice in economic theory. 
In the context of this work, Property $\alpha$ demands that any hypothesis $h$ that is chosen from $\fc$ must also be chosen from $\gc$, if $\gc$ is a contraction of $\fc$ that also contains $h$. 
For example, suppose that $\fc$ is composed of all polynomials of degree smaller than $p$ and $\gc$ consists of all polynomials of degree smaller than $q$ where $q \leq p$, i.e., $\gc\subseteq\fc$. 
Then, if $\algo_{\xi}$ chooses the polynomial of degree $t < q \leq p$ from $\fc$ as a solution, it must also be chosen again from $\gc$ by $\algo_{\xi}$.
\Cref{subfig:alpha-violation} illustrates a situation that violates Property $\alpha$.
Property $\beta$, albeit less intuitive, can be interpreted as follows. 
If there is the hypothesis $h$ that is chosen by $\algo_{\xi}$ from $\gc$ and subsequently from $\fc$, which is an expansion of $\gc$, then all other hypotheses that are considered \emph{equally good}\footnote{In this work, $h,g \in \hsc$ are considered by $\algo_{\xi}$ as \emph{equally good} in $\hsc$ if and only if $h,g \in \algo_{\xi}(\hsc)$.} to $h$ in $\gc$ must also be chosen by $\algo_{\xi}$ from $\fc$. 
For instance, let $f,g \in \gc$ be polynomials of degree $t$ with different coefficients. 
If $\algo_{\xi}$ chooses both $f$ and $g$ from $\gc$, and $f$ from $\fc$, then it must also choose $g$ from $\fc$.
Unlike the first condition, it characterizes the behaviour of the algorithm under the expansion of hypothesis space. \Cref{subfig:beta-violation} illustrates a situation in which this property is violated.

Although rationality of choice has been studied from different angles in economic theory \citep[Ch. A2]{Sen17:CCSW}, its applications to algorithmic choice are scarcer. While other characterizations such as Property $\gamma$ \citep[pp. 318]{Sen17:CCSW}, Property $\delta$ \citep[pp. 320]{Sen17:CCSW}, and choice coherence \citep[Def. 1.1]{Kreps12:MicroEcon} also exist, we chose the $\alpha$ and $\beta$ properties for two reasons. 
Firstly, contraction and expansion of the hypothesis class are the most common scenarios encountered in machine learning. 
For example, adjusting the values of some hyperparameters such as kernel bandwidth and number of hidden nodes will either contract or expand the hypothesis class. 
Secondly, while there exist evidences of scenarios in which human decision may violate internal consistency (see, e.g., \citet{Sen93:InternalConsistency} and \citet[Ch. A2]{Sen17:CCSW} for counterexamples), the lack of internal consistency in machines can be attributed to the use of fundamentally flawed learning algorithms, as discussed further below.

\begin{remark}
    We make the following important observations on \Cref{asmp:coherence}.
    \begin{enumerate}
        \item In \Cref{asmp:coherence}, internal consistency is imposed individually on each learning structure $(\hc_{\omega},\hcol,\algo_{\xi})$, i.e., the behaviour of $\algo_\xi$ with respect to the collection $\hcol$ generated by $\hc_\omega$ and $\algo_\xi$. In other words, we impose no restriction whatsoever on how the algorithms $\algo_{\xi}$ and $\algo_{\xi'}$ belonging to two different learning structures should behave across two distinct subsets of hypothesis classes. In fact, it is often desirable that the learning algorithm is model-class dependent. For example, the algorithm given by \eqref{eq:kernel-machines} will generally depend on the choice of the hypothesis class due to its regularization term.

        \item A learning algorithm that violates \Cref{asmp:coherence} can exhibit a behaviour that may be deemed ``irrational''. For example, it is possible that for three distinct hypotheses $f,g,h$ from $\hc_{\omega}$, $\algo_{\xi}(\{f,g\}) = \{f\}$, but $\algo_{\xi}(\{f,g,h\}) = \{g\}$. That is, $\algo_\xi$ chooses $f$ over $g$ from one model class, but chooses $g$ over $f$ in another where the only distinction between them is the presence of an irrelevant hypothesis $h$. Hence, we will refer to any learning algorithms that satisfy \Cref{asmp:coherence} as ``rational'' learning algorithms.
    \end{enumerate}
\end{remark}

\subsection{Risk Minimizers}
\label{sec:risk-minimizers}

It follows immediately that any risk minimizer (RM) satisfies the internal consistency property.
\begin{proposition}\label{prop:risk-minimizers}
Let $\{(\hc_{\omega},\hcol,\algo_{\xi})\,:\, \omega\in\Omega, \xi\in\Xi\}$ be a collection of learning structures for some index sets $\omega\in\Omega$ and  $\xi\in\Xi$. 
Suppose that $\mathbb{A}_\xi$ is a risk minimizer (RM), i.e., 
\begin{equation}\label{eq:risk-minimizer}
    \algo_\xi(\hsc) = \{ h\in\hsc \,:\, r_\xi(h) \leq r_\xi(g) \;\text{ for all }\; g \in \hsc \}, \quad \hsc\in\hcol
\end{equation}
for some real-valued risk functional $r_\xi:\hc_\omega\to\rr$, $\xi\in\Xi$. Then, $\{(\hc_{\omega},\hcol,\algo_{\xi})\,:\, \omega\in\Omega, \xi\in\Xi\}$ satisfies internal consistency.
\end{proposition}

\begin{proof}[Proof of Proposition \ref{prop:risk-minimizers}]
    (Property $\alpha$): For any $\fc,\gc\in\hcol$ such that $\gc \subseteq \fc$, let $h\in\algo_\xi(\fc)$. Thus, $r_\xi(h) \leq r_\xi(g)$ for all $g\in\fc$. Assume that $h\in\gc$ but $h\notin \algo_\xi(\gc)$. This implies that there exists another hypothesis $f\in\gc$ for which $r_\xi(f) < r_\xi(h)$. However, since $f$ is also in $\fc$, it contradicts with $r_\xi(h) \leq r_\xi(g)$ for all $g\in\fc$. Hence, $h$ must also be in $\algo_\xi(\gc)$. 
    (Property $\beta$): For any $h,g\in\gc$, let $h,g\in\algo_\xi(\gc)$. Assume that $h\in\algo_\xi(\fc)$ but $g\notin\algo_\xi(\fc)$. This implies that $r_\xi(h) \leq r_\xi(f)$ for all $f\in\fc$ and $r_\xi(g) > r_\xi(h)$, which contradicts the fact that $g\in\algo_\xi(\gc)$. Hence, $g$ must also be in $\algo_\xi(\fc)$. This implies that $\algo_\xi(\gc)\subseteq\algo_\xi(\fc)$.
    Finally, since this holds uniformly for all $\omega\in\Omega$ and $\xi\in\Xi$, a collection of learning structures $\{(\hc_{\omega},\hcol,\algo_{\xi})\,:\, \omega\in\Omega, \xi\in\Xi\}$ satisfies internal consistency.
\end{proof}

A majority of machine learning algorithms such as empirical risk minimization (ERM) \citep{Vapnik91:ERM}, structural risk minimization (SRM) \citep{Shawe-Taylot96:SRM}, and invariant risk minimization (IRM) \citep{Arjovsky19:IRM,Ahuja20:IRM-Games} fall into this category.
The risk minimizers turn a learning problem into an optimization problem, rendering it solvable by efficient optimization algorithms, e.g., a stochastic gradient descent (SGD).
In what follows, we say that $\algo_\xi$ can be represented by $r_\xi$ if it minimizes $r_\xi$.

\begin{remark}\label{rem:risk-minimizers}
A few remarks on risk minimizers follow.
\begin{enumerate} 
    \item A risk functional that represents $\algo_\xi$ is not unique.
    For any strictly increasing function $c:\rr\to\rr$, $r'_\xi(h) = c(r_\xi(h))$ is a new risk functional that also represents $\algo_\xi$. 
    From learning perspective, they are informationally identical, i.e., $\algo_\xi$ and $\algo_{\xi'}$ behave identically.
    Any two algorithms $\algo_\xi$ and $\algo_{\xi'}$ that can be represented by $r_\xi$ and $r_{\xi'}$ respectively are said to be equivalent if such a strictly increasing function exists.

    \item Invariance to strictly increasing functions makes risk minimizers immune to trivial manipulations (e.g., strategic manipulation and adversarial attacks) and perturbations (e.g., unreliable communication over massive networks) of the risk functional. 
    The Invariance Restriction (IR) condition, introduced in \Cref{sec:heterogeneous-learning}, generalizes this property to a set of $n$-tuples of risk functionals.

    \item When $\hc_{\omega}$ is finite, it has been shown that any $\algo_{\xi}$ that satisfies \Cref{asmp:coherence} must be a risk minimizer \eqref{eq:risk-minimizer}; see, e.g., \citet[Prop. 1.2c]{Kreps12:MicroEcon}. However, this is not always the case for an (uncountably) infinite $\hc_{\omega}$, which is a common scenario in machine learning, at least not without further assumptions on the learning structure; see, e.g., \citet[Sec. 1.5]{Kreps12:MicroEcon}. 
    Hence, a learning algorithm that satisfies internal consistency is not necessarily a risk minimizer. 
\end{enumerate} 
\end{remark}

Next, we characterize an internally consistent learning structure in terms of a binary relation over the model class.

\begin{proposition}\label{prop:revealed-preference}
    Let $(\hc_{\omega},\hcol,\algo_{\xi})$ be a learning structure for some index sets $\omega\in\Omega$ and  $\xi\in\Xi$.
    For every pair $f,g \in \hc_\omega$, let $f \succeq_{\algo_\xi} g$ if and only if  $f\in\algo_\xi(\{f,g\})$. Then, the binary relation $\succeq_{\algo_\xi}$ is complete and transitive\footnote{The binary relation $\succeq$ on $\hsc$ is \emph{complete} if for every pair $f$ and $g$ from $\hsc$, either $f\succeq g$ or $g\succeq f$ (or both). It is \emph{transitive} if $f \succeq g$ and $g \succeq h$ implies $f \succeq h$.} if and only if $(\hc_{\omega},\hcol,\algo_{\xi})$ satisfies the internal consistency property. Moreover, for every $\hsc\in\hcol$,
    \begin{equation*}
        \algo_\xi(\hsc) = \algo_{\succeq_{\algo_\xi}}(\hsc) := \{h\in\hsc \,:\, h \succeq_{\algo_\xi} g \text{ for all } g\in\hsc \}.
    \end{equation*}
\end{proposition}

\begin{proof}[Proof of Proposition \ref{prop:revealed-preference}]
    ($\Rightarrow$) The conclusiveness property implies that $\algo_\xi(\{f,g\})\neq\emptyset$. Hence, $\succeq_{\algo_\xi}$ is complete. If $f\succeq_{\algo_\xi}g$ and $g\succeq_{\algo_\xi}h$, then $f\in\algo_\xi(\{f,g\})$ and $g\in\algo_\xi(\{g,h\})$. By Property $\beta$, if $g\in\algo_\xi(\{f,g,h\})$, then $f\in\algo_\xi(\{f,g,h\})$. Also, if $h\in\algo_\xi(\{f,g,h\})$, then $g\in\algo_\xi(\{f,g,h\})$. Hence, we have $f\in\algo_\xi(\{f,g,h\})$ in any case. By Property $\alpha$, $f\in\algo_\xi(\{f,h\})$ and $f\succeq_{\algo_\xi}h$, which shows that $\succeq_{\algo_\xi}$ is transitive.
    ($\Leftarrow$) Let $\succeq_{\algo_\xi}$ be complete and transitive. 
    By completeness of $\succeq_{\algo_\xi}$, $f\in\algo_{\xi}(\{f,g\})$ or $g\in\algo_{\xi}(\{f,g\})$. By transitivity of $\succeq_{\algo_\xi}$, if $f\succeq_{\algo_\xi}g$ and $g\succeq_{\algo_\xi}h$, then $f\succeq_{\algo_\xi}h$. 
    This implies that $f\in\algo_{\xi}(\{f,g\})$, $g\in\algo_{\xi}(\{g,h\})$, and $f\in\algo_{\xi}(\{f,h\})$. Then, if $f\in\algo_{\xi}(\{f,g,h\})$, we have both $f\in\algo_{\xi}(\{f,g\})$ and $f\in\algo_{\xi}(\{f,h\})$, which shows that $\algo_{\xi}$ satisfies Property $\alpha$. Next, suppose that $\{f,g\}=\algo_{\xi}(\{f,g\})$ and $\{f,h\}=\algo_{\xi}(\{f,g,h\})$, implying that $\algo_{\xi}$ violates Property $\beta$. Since $h\in\algo_{\xi}(\{f,g,h\})$, it follows from Property $\alpha$ that $h\in\algo_{\xi}(\{f,h\})$, implying that $h\succeq_{\algo_\xi} f$. However, this can create an intransitive relation $f\succeq_{\algo_\xi} g \succeq_{\algo_\xi} h \succeq_{\algo_\xi} f$. Hence, $\algo_{\xi}$ must satisfy Property $\beta$.

    Finally, we show that $\algo_\xi(\hsc) = \algo_{\succeq_{\algo_\xi}}(\hsc)$ for every $\hsc\in\hcol$. Assume that $f\in\algo_\xi(\hsc)$. By $\alpha$ Property, we have for every $g\in\hsc$ that $f\in\algo_\xi(\{f,g\})$. This implies that $f\succeq_{\algo_\xi}g$ and thus $f\in\algo_{\succeq_{\algo_\xi}}(\hsc)$. Now, let us assume that $f\neq g$, $f\in\algo_{\succeq_{\algo_\xi}}(\hsc)$, and $g\in\algo_\xi(\hsc)$. Then, $f\in\algo_\xi(\{f,g\})$ and by $\beta$ Property, $f\in\algo_\xi(\hsc)$, which completes the proof.
\end{proof}

The binary relation $\succeq_{\algo_\xi}$ is known as a \emph{revealed preference} of $\algo_\xi$ and this proposition implies that as long as $\algo_\xi$ satisfies internal consistency, $\algo_\xi(\hsc)$ for any $\hsc\in\hcol$ coincides with those obtained from the learning algorithm $\algo_{\succeq_{\algo_\xi}}$ defined in terms of preferences that are revealed by $\algo_\xi$ operated on one- and two-element subsets of hypotheses; see, e.g., \citet{Sen71:RevealedPref} for details on revealed preference. With this structure, there is a dualism between preference and choice from binary menus, i.e., $f \succeq_{\algo_\xi} g$ if and only if $f$ is chosen from $\{f,g\}$ by ${\algo_\xi}$.

\subsection{Summary}
\label{sec:rationality-discussion}

Representing learning algorithms as choice correspondences broadens our analyses to include scenarios involving moral philosophy, value judgments, and cases where the objective function cannot be represented numerically.
In this section, we impose the internal consistency property, a minimal requirement that distinguishes desirable learning algorithms from undesirable ones. This requirement's importance in machine learning can be seen from two perspectives.
First, as noted in \Cref{rem:risk-minimizers}, internally consistent algorithms behave like risk minimizers, the most popular class of learning algorithms, aligning with community standards. 
Second, \Cref{prop:revealed-preference} implies that algorithms violating internal consistency exhibit an incomplete or intransitive (or both) preference relation, indicating their fundamental design flaws.

\section{Collective Learning}
\label{sec:heterogeneous-learning}

Modern applications of AI involve training and deploying machine learning models across heterogeneous and potentially massive environments.\footnote{Throughout this paper, the term ``environments'' is used colloquially to refers to potential sources of heterogeneous data, e.g., individuals, demographic groups, organizations, tasks, mobile phones, stakeholders, modalities, etc.}
In multi-task learning, we have access to data from different tasks and our goal is to design a learning algorithm that can leverage information sharing across tasks.
In domain generalization and OOD generalization, the setting is similar, but the ultimate goal is to learn models that generalize well to previously unseen environments. 
In federated learning, we aim to optimize the model performance over potentially massive network of remote devices under some real-world constraints such as privacy, security, and access rights.
In collaborative learning, we have access to data that come from multiple stakeholders and our goal is to design an  algorithm that can not only learn, but also incentivize them to collaborate, e.g., by truthfully sharing their proprietary data.
In algorithmic fairness, the heterogeneity of data originates from varied effects of an algorithmic model across different demographic groups. 
Hence, equitability of learning algorithms as well as the resulting models is an important criterion that the ML researchers must also take into account.
In multi-modal learning, our goal is to optimize the model performance across multiple modalities such as images, text, and audios.

\begin{figure}[t!]
    \centering
    \resizebox{0.8\textwidth}{!}{
    \begin{tikzpicture}
        \coordinate (F) at (4,0);
        \coordinate (S) at (-4,0);

        \shade[ball color = yellow, opacity = 0.2] (4,0) circle [radius=2cm];
        \shade[ball color = green, opacity = 0.2] (-4,0) circle [radius=2cm];

        \draw (S) circle [line width=4pt, radius=2cm] node[above, yshift=2.2cm] {Environment(s)};
        \draw (F) circle [very thick, radius=2cm]node[above, yshift=2.2cm] {Hypothesis Class $\hsc\in\hcol$};
        
        \draw (3,0.5) circle (2pt) node[below,yshift=-2pt] {$h_0$};
        \draw[red] (5.1,-0.51) node {\Large$\star$};
        \draw[red] (5.3,-0.8) node {$h^*$};
        \filldraw (3.5,1.5) circle (2pt) node[below,yshift=-2pt] {};
        \filldraw (4.5,1.3) circle (2pt) node[below,yshift=-2pt] {};
        \filldraw (5,0.5) circle (2pt) node[below,yshift=-2pt] {};
        \filldraw (3,1.3) circle (2pt) node[below,yshift=-2pt] {};
        \filldraw (4.3,-1.4) circle (2pt) node[below,yshift=-2pt] {};
        \filldraw (3.5,-1.2) circle (2pt) node[below,yshift=-2pt] {};
        \filldraw (3,-0.5) circle (2pt) node[below,yshift=-2pt] {};
        \draw[snake=bumps,segment length=20pt,thick,->] (3.05,0.45) -- node[above,yshift=3pt] {$\algo_{\rp}$} (4.95,-0.45);

        \filldraw (-2.8,0.1) circle (3pt) node[below,yshift=-2pt] {$r_1$};
        \filldraw (-3.5,1.1) circle (3pt) node[below,yshift=-2pt] {$r_3$};
        \filldraw (-4.3,-0.5) circle (3pt) node[below,yshift=-2pt] {$r_4$};
        \filldraw (-5.2,-0.3) circle (3pt) node[below,yshift=-2pt] {$r_5$};
        \filldraw (-4.6,1) circle (3pt) node[below,yshift=-2pt] {$r_n$};
        \filldraw (-3.4,-1) circle (3pt) node[below,yshift=-2pt] {$r_2$};
        
        \draw (-4,-1.5) circle (2pt);
        \draw (-4.8,-1.3) circle (2pt);
        \draw (-3.7,0) circle (2pt);
        \draw (-5.3,0.3) circle (2pt);
        \draw (-4,1.5) circle (2pt);
        
        \node[chef,minimum size=1cm] at (0,2.2) {};
 
        \draw[->,ultra thick, color=blue] (-2,0.5) -- node[above,yshift=0.5em] {$(r_1,\ldots,r_n) \mapsto (\hc,\hcol,\algo_{\rp})$} (2,0.5);
        \draw[->,ultra thick, color=magenta] (2,-0.5) -- node[below,yshift=-0.5em] {$h^* \in \algo(\hsc)$} (-2,-0.5);

    \end{tikzpicture}}
    \caption{In heterogeneous environments, the task of machine learners is to design an aggregation rule that takes a risk profile $(r_1,r_2,\ldots,r_n)$ representing the performance measures of hypotheses across $n$ environments and produces a learning structure $(\hc,\hcol,\algo_{\rp})$. For each hypothesis class $\hsc\in\hcol$, the algorithm $\algo_\rp$ is implemented to choose the best hypotheses from $\hsc$.
    The lower arrow in the figure represents the deployment process.}
    \label{fig:democratic-ml}
\end{figure}

Despite substantially diverse goals, they can be expressed mathematically as an aggregation function
\begin{equation}
    \label{eq:learning-mechanism}
    F: (r_1,r_2,\ldots,r_n) \mapsto (\hc,\hcol,\algo_{\rp}), \quad n \in \mathbb{N},
\end{equation}
where $\rp := (r_1,r_2,\ldots,r_n): \hsp \to \mathbb{R}^n$ denotes a risk profile computed across $n$ distinct environments. 
For each hypothesis $h\in\hsp$, the profile $\mathbf{r}(h)$ corresponds to an $n$-tuple of (empirical) risk functionals evaluated on $h$ where $r_i(h)$ is associated with the $i$-th environment. 
We omit the indexes $\omega$ and $\xi$ from the learning structure $(\hc,\hcol,\algo_{\rp})$ and write $\algo_{\rp}$ to underline that the learning rule uses the risk profile as an input.
Finally, we make no restriction on the domain of $F$, i.e., $F(\mathbf{r})\neq\emptyset$ for all $\mathbf{r}$ and assume further that $F$ is invariant to a permutation of the risk profile.

One can interpret $F$ in \eqref{eq:learning-mechanism} as ways in which machine learners can conceivably blueprint a learning algorithm from the risk profile.
However, its mathematical simplicity does not trivialize the actual learning procedures which can range from a minimization of an aggregated risk, e.g., average risk, in a centralized environment to market mechanisms by which information about data and models must be exchanged between scattered environments.
\Cref{fig:democratic-ml} illustrates this perspective of machine learning.

\subsection{Risk Profile}
\label{sec:risk-profile}

The risk functional \eqref{eq:erm-pre} can be viewed as a process of reducing multidimensional data and prior knowledge into a univariate (cardinal) measure that is easy to optimize.
Hence, the risk profile is the most natural way of eliciting and measuring the expected performance of different hypotheses in the hypothesis space $\hsp$ across $n$ environments. 
Importantly, it captures not only the empirical risk functionals and inductive biases for the learning problems, but also different preferences and incentives across environments.
On the one hand, the risk profile enables the learning algorithms to succinctly compare hypotheses in $\hsp$. 
That is, $f$ is considered at least \emph{as good as} (resp. \emph{strictly better than}) $g$ in the environment $i$ if $r_i(f) \leq r_i(g)$ (resp. $r_i(f) < r_i(g)$).
They are deemed \emph{equally good} if $r_i(f) = r_i(g)$.
On the other hand, such comparisons can be made without revealing or sharing information that constitute the computation of individual risk functionals.
Further, due to an inherent heterogeneity in the nature of observed environments, the risk profile typically exhibits heterogeneity as a result of different sample sizes, loss functions, inductive biases, regularization schemes, user preferences and incentives, for example. This will be elucidated further in \Cref{sec:escaping} and \Cref{sec:implications}.

\subsection{The Axiomatic Method}
\label{sec:axioms}

Statistically speaking, the population risk \eqref{eq:expected-risk} through its approximation \eqref{eq:erm-pre} is the right objective to optimize under the i.i.d. assumption.  
However, in heterogeneous environments, i.e., when $n\geq 2$, it is unclear \emph{what} to optimize.
To deal with this ambiguity, we adopt an axiomatic method and instead define desirable properties of learning algorithms in heterogeneous environments.
We argue that the aggregation function $F$ in \eqref{eq:learning-mechanism} \emph{should} possess the following properties:

\begin{enumerate}
    \item[(PO)] \textbf{Pareto Optimality:} For all $f,g\in\hc$, $r_i(f) < r_i(g)$ for all $i\in[n]$ implies that $\{f\}=\algo_\rp(\{f,g\})$.
    
    \item[(IIH)] \textbf{Independence of Irrelevant Hypotheses:} For any pair of risk profiles $\rp,\rp'$ and any pair of hypotheses $f,g\in\hc$ such that $r_i(f)=r'_i(f)$ and $r_i(g)=r'_i(g)$, $f\in\algo_{\rp}(\{f,g\})$ if and only if $f\in\algo_{\rp'}(\{f,g\})$.
    
    \item[(IR)] \textbf{Invariance Restriction:} For any pair of risk profiles $\mathbf{r},\mathbf{r}'$ for which there exists $(a_1,\ldots,a_n)\in\rr^n$ and $(b_1,\ldots,b_n)\in\rr^n_+$ such that $r_i(h) = a_i + b_ir_i'(h)$ for all $i\in[n]$, $\algo_{\rp}(\hsc) = \algo_{\rp'}(\hsc)$ for all $\hsc\in\hcol$.
    
    \item[(CI)] \textbf{Collective Intelligence:} There exists no $i\in[n]$ such that for all pair $f,g\in\hc$ and for all $\mathbf{r}$ in the domain of $F$, $r_i(f) < r_i(g)$ implies that $\{f\} = \algo_\rp(\{f,g\})$.
\end{enumerate}

First of all, the PO condition requires that the algorithm chooses a non-dominated model $f$ over $g$ if $f$ is strictly better than $g$ in \emph{all} environments. 
It is hard to argue against PO as a desirable property for a learning algorithm.
Surely, if $f$ is unanimously superior to $g$, there is no reason for $\algo_\rp$ to choose $g$ over $f$.
Secondly, the IIH requires that when choosing between two hypotheses, the algorithm should only rely on the relative risks between these two hypotheses across all environments.
Thirdly, the IR demands that the algorithm must be invariant to any transformation of the risk profile that renders it \emph{informationally identical} to the original one.
It generalizes the observation that most learning algorithms in the homogeneous environment are invariant to the strictly increasing transformation of the risk functional (see \Cref{rem:risk-minimizers}), to a set of $n$-tuples of risk functionals.  
Since $r_i$ is informationally identical to $r_i'$ for all $i\in[n]$, the behaviour of the algorithm should remain unchanged.
Lastly, the CI demands that the algorithm should exhibit a collective behaviour in the sense that it leverages information from multiple environments when learning from a hypothesis class $\hc$.

While PO and CI are nearly undisputed, the IIH and IR conditions deserve a scrutiny because they specify which information in the risk profile are deemed ``relevant''. 
The IIH condition primarily emphasizes the importance of \emph{local} information. It says that the algorithm should base its choice over two hypotheses only on their respective risk values, not that of a third hypothesis so long as it leaves the original risk values unchanged. 
While some readers may perceive this condition as overly restrictive, it is challenging (at least from the author's perspective) to envision a machine learning scenario where violating this condition would be desirable. Such a violation would imply that the algorithm is using information beyond the risk profile when comparing different hypotheses. The reliance on irrelevant information can undermine the integrity of the learning system, making it vulnerable to strategic manipulation.

Conversely, the IR condition focuses on the \emph{global} information. It establishes equivalence classes for risk profiles that are informationally identical and instructs the algorithm to treat them as indistinguishable.
As discussed in \Cref{rem:risk-minimizers}, any strictly increasing transformation of the risk function within each environment does not introduce new information that was not already present in the original risk function. 
Consequently, such transformations should be deemed irrelevant.
For example, consider a learning algorithm operating across multiple hospitals. 
The risk profile is computed using patient data obtained from medical equipment which requires a regular re-calibration. However, this re-calibration merely results in a positive affine transformation of the risk profile prior to the calibration.   
The IR condition asserts that the learning algorithm should not be affected by such re-calibration.
While this proposition seems reasonable, it can be hard to enforce it in practice, rendering the algorithm vulnerable to strategic manipulation. 
Another consequence of this requirement is that it restricts the comparability of certain information across environments, implications of which will be further explored in \Cref{sec:escaping}.

In the following, we say that $\algo_\rp$ is \emph{compatible} with the axioms above if the aggregation function $F$ that produces the learning structure $(\hc,\hcol,\algo_\rp)$ satisfies these axioms. 

\subsection{Impossibility of Collective Intelligence}

The following lemma uniquely characterizes the rational learning algorithm that is compatible with the PO, IIH, and IR properties.

\begin{lemma}\label{lem:erm}
For a finite number of two or more environments and at least three distinct hypotheses, a unique learning algorithm (up to permutation of the risk profile) that is internally consistent and is compatible with PO, IIH, and IR simultaneously is of the form  
\begin{equation}\label{eq:erm}
    \algo_\rp(\hsc) = \{h\in\hsc \,:\, r_i(h) \leq r_i(g), \forall g\in\hsc \}, \quad \hsc\in\hcol,
\end{equation}
for some $i\in [n]$.
\end{lemma} 

\Cref{lem:erm} implies that for at least two or more environments, the only rational algorithm that is compatible with PO, IIH, and IR is the ERM, i.e., a risk minimizer that unilaterally optimizes an individual risk function. As it unilaterally learns from a single environment, this algorithm is incompatible with the CI property.\footnote{Under the anonymity of the environments, the ERM algorithm \eqref{eq:erm} is equivalent to what we call a \emph{simultaneous} ERM. This simultaneous ERM independently learns a collection of models by minimizing the risks across all environments without considering any inter-environment information. We thank Gill Blanchard for suggesting this example.}

\begin{theorem}
    \label{thm:impossibility}
    For a finite number of two or more environments and at least three distinct hypotheses, there exists no rational learning algorithm that is compatible with PO, IIH, IR, and CI simultaneously.
\end{theorem}

To summarize, \Cref{lem:erm} and \Cref{thm:impossibility} imply that if we treat PO, IIH, and IR as primitive properties that any learning algorithms in heterogeneous environments must satisfy, the only possibility that can come out of \eqref{eq:learning-mechanism} is an algorithm that unilaterally minimizes an environment-specific risk functional. In other words, it is impossible for machine learners to design learning algorithms that can leverage information across environments unless they are willing to sacrifice at least one of those fundamental properties, namely, internal consistency, PO, IIH, and IR.

It is instructive to understand intuitively how these axioms respectively yield \Cref{lem:erm} and \Cref{thm:impossibility}.
First, we consider all conceivable algorithms defined on $\hsp$. The undesirable ones are then eliminated by demanding that they satisfy some essential properties.
By successively imposing internal consistency, PO, IIH, and IR, the ERM \eqref{eq:erm} in \Cref{lem:erm} remains as the only possibility. 
Adding CI eliminates this only possibility, giving rise to the impossibility result in \Cref{thm:impossibility}.

\subsection{Proofs}
\label{sec:proofs}

This section provides detailed proofs of our main results, which relies heavily on the insights from the original proof of Arrow's General Possibility Theorem \citep{Arrow50:Impossibility} and its simplification in \citet[pp. 286]{Sen17:CCSW}. 

Let $\mathcal{E}$ be a set of environments. 
Suppose that $\algo_\rp$ is internally consistent and is compatible with PO, IIH, and IR. 
Crucial to the proof is the idea of a set $\mathcal{E}$ being ``decisive''.

\begin{definition}[Decisiveness]
\label{def:decisiveness}
    A set of environments $\mathcal{E}$ is said to be \emph{locally decisive} over a pair of hypotheses $f,g$ if $r_e(f) < r_e(g)$ for all $e\in \mathcal{E}$ implies that $\{f\} = \algo_\rp(\{f,g\})$. 
    It is said to be \emph{globally decisive} if it is locally decisive over every pairs of hypotheses.
\end{definition}

The following two intermediate results provide basic properties of decisive set of environments $\mathcal{E}$.

\begin{lemma}
    \label{lem:decisive-spreadness}
    If a set of environments $\mathcal{E}$ is decisive over any pair $\{f,g\}$, then $\mathcal{E}$ is globally decisive.
\end{lemma}

\begin{proof}
    Let $\{p,q\}$ be any other pair of hypotheses that is different from $\{f,g\}$. 
    Assume that in every environment $e$ in $\mathcal{E}$, $r_e(p) < r_e(f)$, $r_e(f) < r_e(g)$, and $r_e(g) < r_e(q)$.
    For all other environments $e'$ not in $\mathcal{E}$, we assume that $r_{e'}(p) < r_{e'}(f)$ and $r_{e'}(g) < r_{e'}(q)$ and leave the remaining relations unspecified.
    By PO condition, $\{p\} = \algo_\rp(\{p,f\})$ and $\{g\} = \algo_\rp(\{q,g\})$.
    By the decisiveness of $\mathcal{E}$ over $\{f,g\}$, we have $\{f\} = \algo_\rp(\{f,g\})$.
    Then, it follows from the transitivity implied by Proposition \ref{prop:revealed-preference} that $\{p\} = \algo(\{p,q\})$.
    By IIH condition, this must be related only to the relation between $p$ and $q$.
    Since we have only specified information in $\mathcal{E}$, $\mathcal{E}$ must be decisive over $\{p,q\}$ and for all other pairs. 
    Hence, $\mathcal{E}$ is globally decisive.
\end{proof}

\begin{lemma}
    \label{lem:decisive-contraction}
    If a set of environments $\mathcal{E}$ consists of more than one element and is decisive, then some proper subset of $\mathcal{E}$ is also decisive.
\end{lemma}

\begin{proof}
    Since there are at least two environments, we can partition $\mathcal{E}$ into two subsets $\mathcal{E}_1$ and $\mathcal{E}_2$. 
    Assume that $r_e(f) < r_e(g)$ and $r_e(f) < r_e(h)$ in every environment $e \in \mathcal{E}_1$ with the relation between $g$ and $h$ unspecified.
    Let $r_{e'}(f) < r_{e'}(g)$ and $r_{e'}(h) < r_{e'}(g)$ in every environment $e' \in \mathcal{E}_2$. 
    By the decisiveness of $\mathcal{E}$, we have $\{f\}=\algo_\rp(\{f,g\})$.
    Now, if $h$ is at least as good as $f$ for some environments over $\{h,f\}$, then we must have $\{h\} = \algo_\rp(\{h,g\})$ for that configuration.
    Since we do not specify relation over $\{g,h\}$ other than those in $\mathcal{E}_2$, and $r_{e'}(h) < r_{e'}(g)$ in $\mathcal{E}_2$, $\mathcal{E}_2$ is decisive over $\{g,h\}$.
    By Lemma \ref{lem:decisive-spreadness}, $\mathcal{E}_2$ must be globally decisive.
    That is, some proper subset of $\mathcal{E}$ is indeed decisive for that particular case.
    To avoid this possibility, we must remove the assumption that $h$ is at least as good as $f$.
    But then $f$ must be better than $h$.
    However, no environment has this relation over $\{f,h\}$ other than those in $\mathcal{E}_1$ where $f$ is better than $h$.
    Clearly, $\mathcal{E}_1$ is decisive over $\{f,h\}$. 
    Thus, by Lemma \ref{lem:decisive-spreadness}, $\mathcal{E}_1$ is globally decisive. 
    So either $\mathcal{E}_1$ or $\mathcal{E}_2$ must be decisive. This completes the proof.
\end{proof}

We are now in a position to prove \Cref{lem:erm} and \Cref{thm:impossibility}.

\begin{proof}[Proof of \Cref{lem:erm}]
    Consider any two risk profiles $\rp$ and $\rp^*$ such that for any $f,g$ and for all $i\in[n]$, $r_i(f) < r_i(g) \Leftrightarrow r_i^*(f) < r_i^*(g)$.
    For every pair $\{f,g\}$, there exists a positive affine transformation $\{\varphi_i\}$ applied to $\rp^*$ such that 
    \begin{equation*}
        r'_i(f) = \varphi_i(r^*_i(f)) = r_i(f) \quad\text{and}\quad r'_i(g) = \varphi_i(r^*_i(g)) = r_i(g) \;\text{ for all }\; i\in[n].
    \end{equation*}
    By IIH condition, $\{f\}=\algo_{\rp}(\{f,g\})$ if and only if $\{f\}=\algo_{\rp'}(\{f,g\})$ and by IR condition, $\{f\} = \algo_{\rp'}(\{f,g\})$ if and only if $\{f\}=\algo_{\rp^*}(\{f,g\})$. 
    Since this holds pair by pair, clearly $\algo_{\rp}(\hsc) = \algo_{\rp'}(\hsc)$ for all $\hsc\in\hcol$. 
    As a result, we can rely on a pairwise comparison of any two hypotheses.
    Next, by the PO condition, the set of all environments $\mathcal{E}$ is decisive. By Lemma \ref{lem:decisive-contraction}, some proper subset of $\mathcal{E}$ must also be decisive. 
    Given that smaller subset of environments, some proper subset of it must also be decisive, and so on. Since the number of environments is finite, the set will eventually contain just a single environment that is decisive. 
    Hence, the only compatible algorithm in this case is the risk minimizer \eqref{eq:erm}.
\end{proof}

\begin{proof}[Proof of Theorem \ref{thm:impossibility}]
    The impossibility result follows because the remaining algorithm in \Cref{lem:erm} violates the CI condition.
\end{proof}

\subsection{Summary}

To understand the implications of our main results, let us think of machine learning as culinary arts and ML researchers as a chef. 
The space $\hsp$ consists of all conceivable dishes and the chefs have the risk profile $(r_1,\ldots,r_n)$ at their disposal as $n$ different ingredients to develop a new recipe $(\hc,\hcol,\algo_\rp)$.
From any feasible menu $\hsc\in\hcol$, $\algo_\rp(\hsc)$ consists of the dishes cooked from this recipe; see \Cref{fig:democratic-ml}.
To ensure culinary excellence, the Michelin guide might act as a regulatory body by regulating that the chefs adhere to PO, IIH, and IR when creating a new (internally consistent) recipe. 
\Cref{lem:erm} implies that under such regulation, the chefs are restricted to the recipe that can use only one out of $n$ ingredients.
In other words, it is impossible for the chefs to simultaneously adhere to such regulation and create a recipe that can mix multiple ingredients, as implied by \Cref{thm:impossibility}.

\section{Possibility of Collective Intelligence}
\label{sec:escaping}

Although the result of \Cref{thm:impossibility} may initially seem disappointing or discouraging, it highlights a crucial trade-off necessary for collective intelligence: \emph{we must give up at least one of the conditions in Theorem \ref{thm:impossibility} when designing a learning algorithm}. 
The first escape is to remove the internal consistency.
What this means is that there exists a learning structure $(H_\omega,B,\algo_\xi)$ for which the algorithm $\algo_\xi$ is hypothesis-class dependent.  
In other words, the index sets $\Omega$ and $\Xi$ are \emph{incomplete} as they cannot fully characterize the behaviour of the learning algorithm, making model selection extremely hard, limiting practical applications, or suggesting a design flaw.
The second escape is to restrict the domain of $F$. 
In fact, existing assumptions such as task relatedness in multi-task learning and (causal) invariance in domain generalization are domain restriction in disguise. 
That is, it somehow amounts to assuming that there is an invariant structure that is shared across all environments, implying a positive correlation of risk profile. 
The drawback of domain restriction however is that it is normally non-trivial to test whether this condition holds or not in practice. 
Pareto optimality \citep{Pareto1897:Pareto} is a simple and highly appealing criterion of comparison of hypotheses in the multi-objective setting which generalizes the notion of ``minimum risk''. Therefore, the consequence of dropping PO as a necessary criterion for machine learning in general must be immense. 
It also implies that the information contained in the risk profile is not sufficient for learning and some ``irrelevant'' information must be used. Hence, a violation of PO requires some caution.
Dropping IIH opens up a number of possibilities, but also poses similar concern on the use of irrelevant information, implying that the learning algorithms can be susceptible to strategic manipulations.

How about the IR condition? As apparent in the proof of Theorem \ref{thm:impossibility}, this condition restricts information that can be shared across environments to relative ranking between any two hypotheses.
Is this too restrictive?

\paragraph{Informational incomparability.}
To answer this question, consider two hypotheses $h$ and $h'$ from $\hsp$ and risk functionals $r_i$ and $r_j$ from the same risk profile. 
Suppose that $r_i(h') - r_i(h) = r_j(h') - r_j(h) < 0$, i.e., $h'$ is better than $h$ in both environments $i$ and $j$ and by the same margin. 
Then, our ability to relax the IR condition will depend on whether or not we can say ``\emph{$h'$ leads to the same improvement over $h$ in environment $i$ as it does in environment $j$}''.
For instance, will the COVID-19 AI diagnosis system $h'$ lead to the same improvement over the old system $h$ for Johns Hopkins Hospital in Baltimore as it does for Siriraj Hospital in Thailand? Will the new autocorrection system $h'$ lead to the same improvement over the existing one $h$ in terms of satisfaction for users in Japan as it does for users in South Africa? Will the updated face recognition system $h'$ lead to the same improvement over the existing one $h$ for white people as it does for black people? and so on. 
If the answer to these questions is yes, then information beyond relative rankings is usable by the algorithm, which opens up a number of possibilities. For example, summing up a risk profile, i.e., $\sum_{i=1}^n r_i(h)$, leads to a meaningful measure of model performance across different environments.
On the other hand, if we cannot answer these questions with an affirmative yes, then the IR condition must still be in place. 
The reason is that we may not possess sufficient information to make a meaningful comparison between environments beyond the relative rankings of hypotheses, a shortcoming that we decoratively call \emph{informational incomparability}.\footnote{In case of cardinal utility functions, this problem is known in economics as an interpersonal incomparability of utility; see, e.g., \citet[Ch. 7]{Sen17:CCSW}. There has been a long debate on whether one can make a meaningful comparison of welfare of different individuals.}

The informational incomparability corresponds to the extent to which the risk functions reflect the actual performance of the algorithmic model. In scenarios where the loss function are identical across environments, e.g., 0-1 loss or cross-entropy loss, the risk functionals become comparable and the IR condition can be removed. However, this can be too restrictive as it prevents us from considering scenarios where the loss functions differ across environments. Furthermore, inductive biases and regularization schemes that encode specific information about the environments can give rise to the information incomparability. This challenge is further compounded by regulations on AI, privacy, and data governance, e.g., GDPR, EU AI Act, and Digital Market Act (DMA), which limit the transferability of the data across environments (subsidies, hospitals, countries, etc).

Thus, it seems that before we can build generalizable, fair, and democratic learning algorithms in heterogeneous environments, \emph{the first question we must ask is whether we know enough to make meaningful comparisons between them}. 
Some immediate challenges are already in sight. 
The first challenge is a physical one. In federated learning, for example, it is physically impossible to share all the data across a massive network of mobile devices. 
Matters pertaining to privacy, security, and access rights will also limit data sharing across environments.
The second challenge is a cultural one. 
An algorithmic model might have varied effects across different demographic groups simply because of the culture differences.
It is impossible to tell all the differences between any two cultures. 
In algorithmic fairness, for example, there can be a mismatch between measurement modelling and operationalization of social constructs, i.e., abstractions that describe phenomena of theoretical interest such as socioeconomic status and risk of recidivism, which makes it difficult to meaningfully compare different operationalizations \citep{Jacobs21:Measurement}.
The third challenge is of subjective matter. 
In multi-modal learning, the relationship between modalities is often open-ended or subjective \citep{Ramachandram17:Deep-Multi-Modal,Baltrusaitis19:Multi-Modal-Survey}. 
Language is often seen as symbolic, but audio and visual data are represented as signals.
Moreover, likelihood functions defined on different data types are generally incomparable \citep{Javaloy22:Multi-VAE}.
Last but not least, the obstacle can simply be a legal one. 
To protect its people, a government might regulate what kind of and to what extent information can be shared. Well-known examples of this attempt are the EU's General Data Protection Regulation (GDPR) and its upcoming AI Act.\footnote{\url{https://digital-strategy.ec.europa.eu/en/policies/european-approach-artificial-intelligence}}




\section{Practical Implications}
\label{sec:implications}

This section elucidates some connections to social choice theory and discusses direct implications of our main result on several sub-fields of machine learning.

\subsection{Social Choice Theory}

Our result is a reincarnation of the Arrow's Impossibility Theorem \citep{Arrow50:Impossibility,Sen17:CCSW} which forms the basis of modern social choice theory; see, e.g., \citet{Patty19:SocialChoice} and references therein. 
To understand this, suppose that $\hsp$ consists of a finite number of at least three hypotheses representing a set of alternatives.
In Arrow's setting, he is interested in the social welfare function (SWF):
$$F: (\succeq_1,\succeq_2,\ldots,\succeq_n) \,\mapsto\, \succeq$$ 
that aggregates preferences of $n$ individuals $\succeq_1,\ldots,\succeq_n$ over the set of alternatives to obtain the \emph{social} preference $\succeq$. He also demands $\succeq$ to be \emph{rational}, i.e., complete and transitive. 
The impossibility result is established under similar set of axioms, namely, Universal Domain (no restriction on the domain of $F$), Pareto Principle (our PO), Independence of Irrelevant Alternatives (our IIH), and Non-dictatorship (our CI). 
Our aggregation function $F$ in \eqref{eq:learning-mechanism} is similar to the SWF in the sense that our IR condition further restricts usable information of the risk profile to relative rankings of hypotheses, as is apparent in the proof of \Cref{lem:erm} (cf. \Cref{sec:proofs}).
Moreover, as shown in \Cref{prop:revealed-preference}, the internal consistency ensures that the revealed preference of $\algo_\rp$ will be complete and transitive.

The most closely related to our work is \citet{Sen70:CCSW} which considers the social welfare functional (SWFL) 
$$F: (u_1,u_2,\ldots,u_n) \,\mapsto\, C(\cdot)$$ 
that aggregates cardinal utility functions of $n$ individuals to obtain the social choice function (SCF) $C(\cdot)$ over the set of alternatives. 
By weakening some restrictions on $C(\cdot)$, possibility results start to emerge; see, e.g., \citet[T.A2*.1, pp. 316]{Sen17:CCSW}. 
In fact, these connections are natural once one realizes that as soon as algorithmic models have societal impact, modern machine learning becomes a social choice problem, as succinctly put by the best-selling author Brian Christian in \emph{The Alignment Problem: Machine Learning and Human Values}: ``[...] every machine-learning system \emph{is} a kind of parliament, in which the training data represent some larger electorate--and, as in any democracy, it's crucial to ensure that everyone gets a vote.'' \citep[pp. 33]{Christian20:Alignment}.
More broadly, we argue that this sort of \emph{aggregation} is endemic to all learning problems in heterogeneous environments.

Nevertheless, a few distinctions deserve further discussion. 
First of all, in modern machine learning we almost always have to deal with the infinite hypothesis spaces. 
The fact that the impossibility still persists even when the size of $\hsp$ becomes infinite suggests that neither adding more data nor scaling up the models alone will get us out of this roadblock. 
On the contrary, the increase in the number of alternatives may even lead to other impossibility results; see, e.g., Theorem 4*2 and 4*3 in \citet{Sen17:CCSW}.  
Secondly, choice correspondence is a basis of microeconomic theory in that it is a classical representation of human choice behaviour \citep[Chapter One]{Kreps12:MicroEcon}. In this work, we instead use it to characterize the behaviour of learning algorithms. 
Although this work imposes similar behavioural rationality, which has been deemed unrealistic by some behavioural economists and social scientists, this distinction implies that there is more room for reasoning about how learning algorithms \emph{should} behave.
Lastly, our work differs from the textbook machine learning in that we start from all conceivable learning algorithms and then rule out those that are incompatible with the desirable properties until we arrive at the unique algorithm, i.e., \Cref{lem:erm}, or none at all, i.e., \Cref{thm:impossibility}.

Lastly, the axiomatic approach is gaining traction in the mainstream machine learning. For example, it has been used to characterize algorithms for clustering \citep{Kleinberg02:Clustering}, network analysis \citep{van-den-Brink03:Ranking}, algorithmic fairness \citep{Williamson19:FRM}, and multi-task learning \citep{Navon22:MTL-Bargaining}, for example. Furthermore, several papers have also explored modern applications of social choice and machine learning \citep{Xia13:SCM-ML,Xia20:Smoothed-SC} including human-AI alignment \citep{Conitzer24:AI-Alignment} and multi-task benchmarks \citep{Zhang24:MT-Bench}.

\subsection{Multi-source Learning}

Learning from multi-source data has a long history in machine lerning \citep{Cortes21:MSDA,Hoffman18:MSDA,Zhao18:AdversarialMSDA,Blanchard11:Generalize, Muandet13:DG, Mahajan21:DG-CausalMatching, Wang21:DG-Review, Zhou21:DG-Review,Zhang21:MTL-Review,Sener18:MTL-MOO}.
In critical areas like health care, we typically have access to data from $n$ distinct environments, which can be represented by $n$ probability distributions $P_1(X,Y), \ldots, P_n(X,Y)$.
In multiple-source adaptation (MSA) and domain generalization (DG), the risk profile can be expressed in terms of the empirical losses:
$$\mathbf{r}(h) = (\hat{r}_1(h),\ldots,\hat{r}_n(h)) = \left( \frac{1}{m_1}\sum_{k=1}^{m_1}\ell_1(h(x^1_k),y^1_k),\ldots,\frac{1}{m_n}\sum_{k=1}^{m_n}\ell_n(h(x^n_k),y^n_k)\right)$$
where $(x^{i}_k,y^{i}_k)_{k=1}^{m_i}$ denotes a sample of size $m_i$ from $P_i(X,Y)$. 
The empirical losses $(\hat{r}_1(h),\ldots,\hat{r}_n(h))$ measures average performances of the hypothesis $h$ across $n$ environments. 
When there are different tasks across these environments, i.e., multi-task learning (MTL), we can rewrite the risk profile as $\mathbf{r}(\mathbf{h}) = \left(\hat{r}_1(\mathbf{h}),\ldots,\hat{r}_n(\mathbf{h})\right)$ where
\begin{equation*}
\hat{r}_i(\mathbf{h}) := \sum_{j=1}^n\frac{\mathbbm{1}[j=i]}{m_j}\sum_{k=1}^{m_j}\ell(h_j(x^j_k),y^j_k)
\end{equation*}
and $\mathbbm{1}[i=j]=1$ if $i=j$ and zero otherwise, and $\hc^n := \bigtimes_{i=1}^n \hc_i$ such that for each $\mathbf{h}$ in $\hc^n$, $\mathbf{h} = (h_1,\ldots,h_n)$ where $h_i\in\hc_i$ for $i\in[n]$.

In this case, Theorem \ref{thm:impossibility} implies that under the PO, IIH, and IR conditions, multi-source adaptation and multi-task learning is impossible as the algorithm cannot leverage information across multiple sources or tasks. However, when the loss functions $\ell_1,\ell_2,\ldots,\ell_n$ are identical, the risk functionals $\hat{r}_1,\hat{r}_2,\ldots,\hat{r}_n$ become comparable. As a result, the IR condition can be dropped, allowing for possibilities of learning collectively across tasks.
Similarly, DG algorithms cannot improve upon the standard ERM under the same conditions, which has previously been observed empirically in \citet{Koh21:WILDS} and \citet{Gulrajani21:DG}.
\citet{David10:ImpossibleDA} provides impossibility theorems for domain adaptation (DA) problems ($n=2$) where our impossibility result does not hold.

\subsection{Algorithmic Fairness}

As AI systems become increasingly ubiquitous, societal impact of these systems also become more visible.
To ensure that decisions guided by algorithmic models are equitable, researchers have started paying careful attention to algorithmic bias and unfairness that arise from deploying them in the real world. 

In the field of fair machine learning, myriad formal definitions fairness have been proposed and studied by both computer science and economics communities \citep{Verma18:FairnessExplained,Hutchinson19:Fairness,Mitchell21:Fairness}.
\citet{Dwork12:IndvFairness} calls for the idea that similar individuals should be treated similarly, which requires an appropriate measure of similarity. 
Group-based fairness requires that algorithms have equal errors rate across groups defined by protected attributes such as race and gender \citep{Hardt16:EO,Kleinberg18:Fairness,Zafar19:Fairness,Rambachan20:Fairness,Mitchell21:Fairness}.
Popular fairness criteria include demographic parity, equal of opportunity, and equalized odds, to name a few. 
To promote these fairness criteria, the learning problem is often formulated as a constrained optimization problem and solved using relaxations of the fairness constraints. However, \citet{Lohaus20:TooRelax} demonstrates that relaxations sometimes fail to produce fair solutions.
A number of recent works also explore interventional and counterfactual approaches to mitigating unfairness \citep{Kilbertus17:Avoiding,Kusner17:CounterfactualFairness,Nabi18:FairOutcomes,Chiappa19:PathFairness}

In this context, the risk profile $(r_1(h),r_2(h),\ldots,r_n(h))$ may encode the error rates of algorithmic model $h$ across groups defined by protected attributes. 
When the number of groups is larger than two, \Cref{lem:erm} implies that there will be a single group that is indiscriminately favored by the learning algorithm, which hardly seems fair by any standard. 
Interestingly, this form of unfairness arises even before we start imposing any of the aforementioned fairness constraints. 
In other words, if we consider PO, IIH, and IR as primitive properties, then there is no room left for fairness (and anything else).
Note that our impossibility result differs from that of \citet{Corbett-Davies17:Cost-of-Fairness}, \citet{Chouldechova17Fairness}, and \citet{Kleinberg17:FairTradeoff} which shows the mathematical incompatibility between different fairness criteria.

Similar to our work, recent works have also advocated for preference-based notion of fairness \citep{Zafar17:FairnessBeyond,Dwork18:Decoupled,Ustun19:NoHarm} as well as its welfare-economics interpretation \citep{Hu20:Fairness-Welfare,Mullainathan18:Fairness-SWF}.

\subsection{Federated Learning} 

As deep learning (DL) models keep growing in complexity, we are in need of huge amount of carefully curated data and substantial amount of computational energy for training them  \citep{Strubell19:EnergyNLP,Strubell20:EnergyDL}.
Unfortunately, amalgamating, curating, and maintaining a high-quality data set can take considerable time, effort, and expense.
For example, health data is highly sensitive and its usage is tightly regulated \citep{Rieke20:FL-Health}.
Federated learning (FL) is a decentralized form of machine learning that has emerged as a promising alternative approach for overcoming these challenges \citep{Konecny16:FedOpt,Mcmahan17:FL,Li20:FL,Kairouz21:FL}. 
While FL is designed to overcome data governance and privacy concerns by training ML models \emph{collaboratively} without exchanging the data itself, it also paves the way for democratization of AI, more energy-efficient approaches for training DL models, and positive environmental impact of training large AI models.
For instance, it has been shown that FL can lead to lower carbon emission than traditional learning \citep{Qiu20:FL-Carbon}.

The main assumption of FL is that there exist private data sets $Z_1,Z_2,\ldots,Z_n$ residing at $n$ local nodes (e.g., mobile phones, hospitals, planets, or galaxies\footnote{We envision an intergalactic learning in the near future.}).
The goal is then to train DL models on the entire data set $Z=\{Z_1,\ldots,Z_n\}$ while ensuring that each of them never leaves its local node.
Formally, let $r$ denotes a global loss functional obtained via a weighted combination of $n$ local losses $r_1,\ldots,r_n$ computed from local data sets $Z_1,\ldots,Z_n$:
\begin{equation*}
    \min_{h\in\hc}\,r(Z;h)\quad \text{with} \quad
    r(Z;h) := \sum_{i=1}^nw_ir_i(Z_i;h),
\end{equation*}
where $w_i > 0$ denote the respective weight coefficients; see, also, \citet{Li21:Tilted} and \citet{Li21:Tilted-ML} for alternative loss functionals inspired by fair resource allocation \citep{Moulin03:FairDivision}.
One of the most popular FL algorithms, \texttt{FedAvg} \citep{Mcmahan17:FL}, typically works by first initializing a global model and broadcasting it to local nodes. 
The local nodes update the model by executing the training on local data. The model updates, e.g., parameters and gradients, are subsequently sent back to the server where they are aggregated to update the global model. The process is repeated until convergence. 
It is not difficult to see that this training process is an instance of the aggregation rule \eqref{eq:learning-mechanism}.

FL is strikingly similar to a voting system, which is one of the most studied scenarios in social choice theory.
To understand this connection, let $\hc$ be a set of candidates and local nodes represent $n$ voters. 
In each round of voting, voters cast the votes by submitting their preferences in the form of locally best parameters or gradient updates. 
The server then aggregates these preferences to obtain the globally best candidate. 
The common limiting factor of both FL and voting system is that information about voters cannot be revealed beyond their preferences.
From this perspective, it is unsurprising that the same patterns of inconsistencies that have previously been observed in the voting systems would also arise in the FL setting.

Our impossibility result suggests that any FL algorithms must violate at least one of the PO, IIH, and IR conditions.
Otherwise, the founding principle of FL cannot be fulfilled.

\subsection{Multi-modal Learning and Heterogeneous Data} 

Multi-modal machine learning has seen much progress in the past few years \citep{Ngiam11:Multi-Modal-DL,Ramachandram17:Deep-Multi-Modal,Baltrusaitis19:Multi-Modal-Survey}.
Its goal is to build models that can process and relate information from multiple modalities such as images, texts, and audios.
Furthermore, heterogeneous data are also increasingly common \citep{Nazabal20:HI-VAE,Valera20:LFM-Het}. 
For example, human-centric data like the Electronic Health Record (EHR) are composed of attributes that have different formats including discrete (e.g., gender and race), continuous (e.g., salary), and positive count data (e.g., blood counts) among others.

Formally, let $\X$ be a data space that can be partitioned into $n$ different modalities as $\X=\bigtimes_{i=1}^n\X_i$ and $\hc^n = \bigtimes_{i=1}^n\hc_i$ denotes the corresponding hypothesis class.
The sub-class $\hc_i$ is the hypothesis class associated with the input space $\X_i$.
In this case, we can express the risk profile as $\mathbf{r}(\mathbf{h}) = \left(\hat{r}_1(\mathbf{h}),\ldots,\hat{r}_n(\mathbf{h})\right)$ where
$$\hat{r}_i(\mathbf{h}) := \sum_{j=1}^n \mathbbm{1}[j=i]c_j(x^j\,|\,h_j)$$
and $\mathbf{h} = (h_1,\ldots,h_n)\in\hc^n$.
Here, for each $\mathbf{x} = (x^1,\ldots,x^n) \in \X$, $x^j \in \X_j$ for $j\in[n]$ and $c_j(x^j\,|\,h_j)$ denotes a score function, e.g., negative log-likelihood function, associated with the $j$-th modality. 
Hence, multi-modal learning can be viewed as an aggregation rule $F$ in \eqref{eq:learning-mechanism} where the learning structure is defined on the compound hypothesis class $\hc^n$.
Hence, the heterogeneity of multi-modal data makes it particularly challenging for coordinated and joint representation learning, especially when the PO, IIH, and IR must be imposed on the learning algorithms.

\section{Discussion and Conclusion}
\label{sec:conclusion}

To conclude, we prove the impossibility result for designing a rational learning algorithm that has the ability to successfully learn across heterogeneous environments whether they represent individuals, demographic groups, mobile phones, siloed data from hospitals, or data modalities.
By representing any conceivable algorithm as an internally consistent choice correspondence over a hypothesis space, we provide reasonable-looking axioms that are deemed necessary, namely, Pareto Optimality (PO), Independence of Irrelevant Hypotheses (IIH), and Invariance Restriction (IR).
The unique algorithm compatible with all of the axioms turns out to be the standard empirical risk minimization (ERM) that unilaterally learns from a single arbitrary environment. 
This possibility result implies the impossibility of \emph{Collective Intelligence} (CI), the algorithm's ability to successfully learn across heterogeneous environments.
Our general impossibility theorem elucidates the fundamental trade-off in emerging areas of machine learning such as OOD generalization, federated learning, algorithmic fairness, and multi-modal learning.

More importantly, this result reveals a subtle challenge, which we decoratively call \emph{informational incomparability}, that is hard for the ML researchers to overcome.
The true challenge of learning in heterogeneous environments is the heterogeneity itself.
Unlike in the homogeneous environment, relative impacts of algorithmic models in heterogeneous environments could vary in ways that cannot be measured precisely by the risk functionals due to physical constraints, culture differences, or ethical and legal concerns.
As a result, comparative information beyond the relative rankings of any two models cannot be leveraged by the algorithm.
Learning algorithms that disregard this condition allow irrelevant information to influence their outcomes, rendering the entire systems susceptible to  strategic manipulation.
To make progress, it is thus imperative not only to strengthen privacy and information security such that information can be disseminated securely, but also to better understand the real impact of algorithmic models in deployment.

Our work made a number of simplifying assumptions.
First of all, by modeling learning algorithms as choice correspondences, some crucial aspects such as initialization strategies, model architectures, data augmentation, optimization methods, and regularization strategies are abstracted away. 
With an increasing number of new algorithms proposed every year, finding the optimal framework that captures the right kind of behaviors while disregarding negligible details is one of the important future directions.
In particular, generalizations of the internal consistency property might shed light into more sophisticated learning behaviors in CNN, RNN, and transformer, for example.
Second, while our two-stage model presented in \Cref{sec:two-stage-model} subsumes a model selection procedure, this process is implicit in the collective learning \eqref{eq:learning-mechanism}. 
While this broader perspective enables us to make a more general statement about the entire pipeline, it would be valuable to investigate the rationality conditions governing model selection structures, drawing parallels to our analysis of the primitive learning structure.
Third, generalization to unseen data is another aspect that we omit in this work. 
Nevertheless, we have somehow demonstrate the generalization ability of ERM as a learning algorithm in the sense that it remains invariant under the same set of axioms regardless of the nature of the environments in which it operates.
Last but not least, it remains to explore whether other impossibility results can be established.
Notably, given a growing interest in learning under strategic behaviours and adversarial examples, one of the future directions is to generalize the Gibbard-Satterthwaite theorem \citep{Gibbard73:StrategyProof,Satterthwaite75:StrategyProof} which shows that there exists no aggregation rule that is strategy-proof.

\clearpage
\putbib[refs]
\end{bibunit}


\end{document}